\documentclass[10pt,journal,compsoc]{IEEEtran}
\usepackage{amsmath,amsfonts}
\usepackage{bbm}
\usepackage{algorithmic}
\usepackage{array}
\usepackage{textcomp}
\usepackage{stfloats}
\usepackage[hyphens]{url}
\usepackage{verbatim}
\usepackage{graphicx}
\usepackage{authblk}
\usepackage{color}
\usepackage{longtable}
\usepackage{caption}
\usepackage{subcaption}
\usepackage{amsthm}
\usepackage{multirow}
\usepackage{multicol}

\usepackage{cite}

\usepackage{amsmath, amsfonts,graphicx,algorithm,algorithmic,mathrsfs,enumitem,xfrac}
\usepackage[ruled,vlined,linesnumbered,algo2e]{algorithm2e}
\usepackage[top=1.5cm, bottom=2.5cm, left=1.5cm, right=1.5cm]{geometry}
\usepackage[dvipsnames,svgnames,table]{xcolor}
\usepackage{tcolorbox}
\usepackage{booktabs}
\usepackage{color}
\usepackage{xcolor}
\usepackage{bbm}
\usepackage{xspace}

\usepackage{graphicx}
\newtheorem{theorem}{Theorem}

\def\BibTeX{{\rm B\kern-.05em{\sc i\kern-.025em b}\kern-.08em
    T\kern-.1667em\lower.7ex\hbox{E}\kern-.125emX}}
\usepackage{balance}

\usepackage[textwidth=1cm,textsize=tiny]{todonotes}

\usepackage{hyperref}

\newcommand{\vct}[1]{\ensuremath{\boldsymbol{#1}}}

\newcommand{\set}[1]{\ensuremath{\mathcal{#1}}}

\newcommand{\argmax}{\operatornamewithlimits{\arg\,\max}}

\newcommand{\argmin}{\operatornamewithlimits{\arg\,\min}}

\newcommand{\myparagraph}[1]{\noindent \textbf{#1}}

\newcommand{\ie}{i.e.,\xspace}
\newcommand{\eg}{e.g.,\xspace}

\newcommand{\wrt}{w.r.t.\xspace}

\newcommand{\edit}[1]{\textcolor{black}{#1}}

\newcommand{\anf}{NF\xspace}
\newcommand{\anfs}{NFs\xspace}
\newcommand{\rnf}{RNF\xspace}
\newcommand{\rnfs}{RNFs\xspace}

\newcommand{\pct}{PCT\xspace}
\newcommand{\pctat}{PCAT\xspace}
\newcommand{\rfat}{RCAT\xspace}
\newcommand{\rcat}{RCAT\xspace}
\newcommand{\mixmseat}{RCAT\xspace} 

\newcommand{\numepochs}{12\xspace}
\newcommand{\batchsize}{500\xspace}
\newcommand{\lr}{$10^{-3}$\xspace}

\newcommand{\fsrc}{\ensuremath{f^{\rm src}}\xspace}
\newcommand{\f}{\ensuremath{f}\xspace}
\newcommand{\fnew}{\ensuremath{f^{\rm new}}\xspace}
\newcommand{\fold}{\ensuremath{f^{\rm old}}\xspace}
\newcommand{\naive}{na\"ive\xspace}
\newcommand{\base}{baseline\xspace}

\newcommand{\trainset}{\ensuremath{\set D}\xspace}

\begin{document}

\title{Robustness-Congruent Adversarial Training for Secure Machine Learning Model Updates}

\author[1]{Daniele Angioni}
\author[2]{Luca Demetrio}
\author[1]{Maura Pintor}
\author[2]{Luca Oneto}
\author[2]{Davide Anguita,~\IEEEmembership{Senior Member,~IEEE}}
\author[1]{Battista~Biggio,~\IEEEmembership{Fellow,~IEEE}}
\author[1,2]{Fabio Roli,~\IEEEmembership{Fellow,~IEEE}}

\affil[1]{Department of Electrical and Electronic Engineering,
University of Cagliari, Italy}
\affil[2]{Department of Informatics, Bioengineering, Robotics and Systems Engineering, University of Genova, Italy}

\markboth{Journal of \LaTeX\ Class Files,~Vol.~18, No.~9, September~2020}%
{Regression-Free}

\IEEEtitleabstractindextext{
\begin{abstract} 
Machine-learning models demand periodic updates to improve their average accuracy, exploiting novel architectures and additional data. However, a newly updated model may commit mistakes the previous model did not make. Such misclassifications are referred to as \textit{negative flips}, experienced by users as a regression of performance. 
In this work, we show that this problem also affects robustness to adversarial examples,  hindering the development of secure model update practices. In particular, when updating a model to improve its adversarial robustness, previously ineffective adversarial attacks on some inputs may become successful, causing a regression in the perceived security of the system.
We propose a novel technique, named robustness-congruent adversarial training, to address this issue. It amounts to fine-tuning a model with adversarial training, while constraining it to retain higher robustness on the samples for which no adversarial example was found before the update. We show that our algorithm and, more generally, learning with non-regression constraints, provides a theoretically-grounded framework to train consistent estimators. 
Our experiments on robust models for computer vision confirm that both accuracy and robustness, even if improved after model update, can be affected by negative flips, and our robustness-congruent adversarial training can mitigate the problem, outperforming competing baseline methods.
\end{abstract}

}

\maketitle

\IEEEdisplaynontitleabstractindextext

\IEEEpeerreviewmaketitle

\section{Introduction}

Many modern machine learning applications require frequent model updates to keep pace with the introduction of novel and more powerful architectures, as well as with changes in the underlying data distribution.
For instance, when dealing with cybersecurity-related tasks like malware detection, novel threats are discovered at a high pace, and machine learning models need to be constantly retrained to learn to detect them with high accuracy.
Another example is given by image tagging, in which image classification and detection models are used to tag pictures of users, and the variety of depicted objects and scenarios varies over time, requiring constant updates.
In both cases, as novel and more powerful machine learning architectures emerge, they are rapidly adopted to improve the average system performance; consider, for instance, the need for transitioning from convolutional neural networks to transformer-based architectures.

Within the aforementioned scenarios, the practice of delivering frequent model updates opens up a new challenge related to the maintenance of machine learning models and their performance as perceived by the end users.
The issue is that average accuracy is not elaborate enough to also account for sample-wise performance. 
In particular, even if average accuracy increases after an update, some samples that were correctly predicted by the previous model might be misclassified after the model update.
There is indeed no guarantee that a newly updated model with higher average accuracy will not commit any mistake on the samples that the previous model correctly predicted.
The samples that the previous model correctly predicted and became misclassified after the update have been referred to as \textit{Negative Flips} (NFs) in~\cite{yan2021positive}. Such mistakes are perceived by end users and practitioners as a \textit{regression} of performance, similarly to what happens in classical software development, where the term ``regression'' refers to the deterioration of performance after an update. For this reason, reducing negative flips when performing model updates can be considered as important in practice as improving the overall system accuracy.
To better understand the relevance of this issue, consider again the case of malware detection. 
Experiencing NFs in this domain amounts to having either previously known legitimate samples misclassified as false
positives, or previously-detected malware samples misclassified as legitimate, increasing the likelihood of infecting devices
in the wild.\footnote{\url{https://www.mandiant.com/resources/blog/churning-out-machine-learning-models-handling-changes-in-model-predictions}}
Similarly, for the case of image tagging, users might find some of their photos changing labels and potentially being mislabeled after a model update, resulting in a bad user experience.
Yan et al.~\cite{yan2021positive} have been the first to highlight this issue, and proposed an approach aimed at minimizing negative flips, 
referred to as \emph{Positive Congruent Training} (\pct). The underlying idea of their method is to include an additional knowledge-distillation loss while training the updated model to retain the behavior of the old model on samples that were correctly classified before the model update. This forces the updated model to reduce the number of errors on such samples (\ie the NFs), while also aiming to improve the average classification accuracy.

In this work, we argue that model updates may not only induce a perceived regression of classification accuracy via negative flips, but also a regression of other trustworthiness-related metrics, including \textit{adversarial robustness}.
Adversarial robustness is the ability of machine learning models to withstand \emph{adversarial examples}, \ie inputs carefully perturbed to mislead classification at test time~\cite{biggio2013evasion,szegedy2014intriguing}.
Recent progress has shown that adversarial robustness can be improved by adopting more recent neural network architectures and data augmentation techniques~\cite{croce2021robustbench}.
When updating the system with a more recent and robust model, one is expected to gain an overall increase in \textit{robust accuracy}, \ie the fraction of samples for which no adversarial example (crafted within a given perturbation budget) is found. However, similarly to the case of classification accuracy, previously ineffective adversarial attacks on some samples may be able to evade the newly updated model, thereby causing a perceived regression of robustness. In the remainder of this manuscript, we refer to these newly induced mistakes as \textit{robustness negative flips} (RNFs). We discuss the different types of regression in machine learning in Sect.~\ref{sect:regression}.

Our contribution is twofold. We are the first to show that the update of robust machine-learning models can cause a perceived regression of their robustness, meaning that the new model may be fooled by adversarial examples that the previous model was robust to.
Furthermore, we propose a novel technique named \emph{robustness-congruent adversarial training} (\rfat) to overcome this issue, presented in Sect.~\ref{sect:rct}.
As the name suggests, our methodology utilizes the well-known adversarial training (AT) procedure to update machine learning models by incorporating adversarial examples within their training data~\cite{madry2017towards}. In particular, we enrich AT by re-formulating the optimization problem with an additional non-regression penalty term that forces the model to retain high robustness on the training samples for which no adversarial example was found, minimizing the fraction of robustness negative flips (as well as the fraction of negative flips). Finally, we show that our technique is also theoretically grounded, demonstrating that learning with a non-regression constraint provides a statistically-consistent estimator, without even affecting the usual convergence rate of $O(\sfrac{1}{\sqrt{n}})$, where $n$ is the number of training samples.

Our experiments, reported in Sect.~\ref{sect:exp}, confirm the presence of regression when updating robust models, showing that state-of-the-art image classifiers~\cite{croce2021robustbench} with higher robustness than their predecessors are misled by some previously-detected adversarial examples.
We further show that, when updating models using \rfat, the accuracy and robustness of the updated models improve while containing both negative flips and robustness negative flips, outperforming competing baselines.
We discuss related work on backward compatibility of machine-learning models and continual learning in Sect.~\ref{sect:related}, and conclude the paper by discussing limitations of the current approach and future research directions in Sect.~\ref{sect:concl}.
The code to reproduce all experiments and results is available at: \url{https://github.com/pralab/robustness-congruent-advtrain}.

\section{Regression of Machine Learning Models}
\label{sect:regression}

\begin{figure*}[t]
    \centering
    \includegraphics[width=0.99\textwidth]{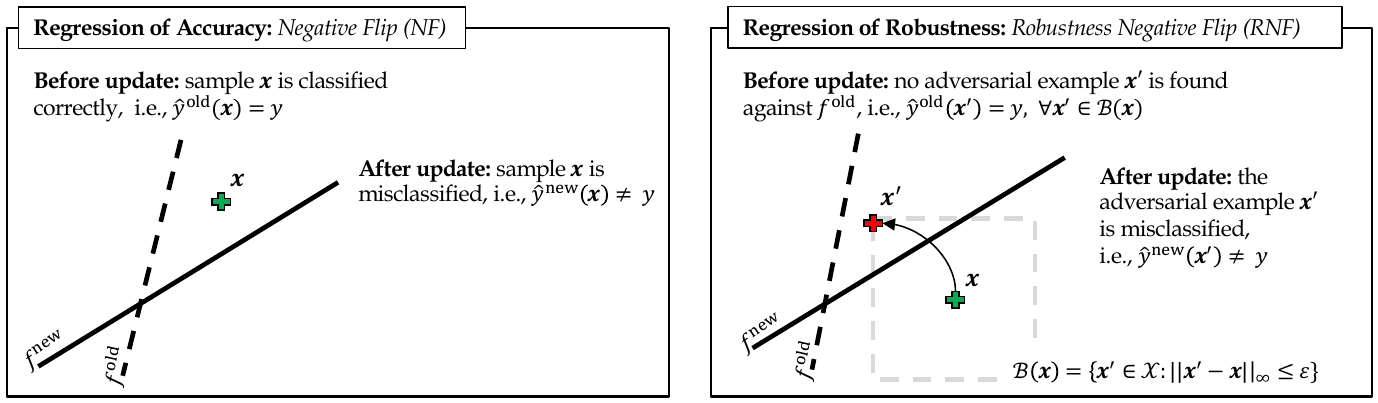}
    \caption{Regression modes in machine-learning model updates. \textit{Left}: Regression of accuracy induced by negative flips (NFs). When updating an old model \fold (dashed black line) with a new model \fnew (solid black line), a test sample $\vct x$ that was correctly classified by \fold may be misclassified by \fnew, causing an NF. 
     \textit{Right}: Regression of robustness induced by robustness negative flips (RNFs). In a different setting, the test sample $\vct x$ may still be correctly classified by \fnew. However, while no adversarial examples are found against \fold (since the perturbation domain $\set B(\vct x)$, represented by the dashed gray box around $\vct x$, never intersects the decision boundary of \fold),  an adversarial example $\vct x^\prime$ is found against \fnew, causing an RNF.}
    \label{fig:types_nf}
\end{figure*}

Before delving into the discussion about the types of regression that may be experienced when updating machine learning models, we define some basic notation that will be used throughout this manuscript.

\myparagraph{Notation.} Let us denote with $\set D = (\vct x_i, y_i)_{i=1}^n$ the training set, consisting of $n$ $d$-dimensional samples $\vct x_i \in \set X \subseteq \mathbb R^d$,\footnote{When data is normalized, as for images, typically $\vct x \in \set X = [0, 1]^d$.} along with their class labels $y_i \in \set Y = \{1, \ldots, c\}$. 
We similarly define the test set as $\set T = (\vct x_j, y_j)_{j=1}^m$, and assume that both datasets are sampled from the same, unknown probability distribution. Accordingly, a machine learning model can be represented as a function $f : \set X \mapsto \mathbb R^c$, which outputs a confidence value for each class (also referred to as \textit{logit}). For compactness, we will also use $f_k(\vct x)$ to denote the confidence value associated to class~$k$. The predicted class label can be thus denoted with $\hat{y}(\vct x) = \arg\max_{k \in \set Y} f_k(\vct x)$, where the dependency of $\hat{y}$ on $\vct x$ may be omitted in some cases to keep the notation uncluttered.
Training a model amounts to finding a function $f$ within a feasible set $\set F$ (\eg constrained by a fixed network architecture) that minimizes a loss $L(y, f(\vct x))$ on the training samples in $\set D$, typically using a convex loss that provides an upper bound on the classification error (\ie the zero-one loss), while penalizing complex solutions via a regularization term that prevents overfitting. This helps find solutions that achieve lower error on the test set $\set T$ and are thus expected to generalize better.

\subsection{Regression of Accuracy: Negative Flips (NFs)} 

Training a model by minimizing the aforementioned objective means finding a suitable set of parameters for \f that reduces the probability of misclassification over all samples. However, this formulation does not constrain the behavior of the model on specific samples or regions of the feature space. Thus, when updating an old model \fold with a new model \fnew that is trained on the same data $\set D$ but using a different network architecture, one may observe that \fnew yields a lower classification error on $\set T$ while misclassifying different samples than those that were misclassified by \fold. 
In other words, there might be some test samples $\vct x$ for which $\hat y^{\rm old}(\vct x) = y$ and $\hat y^{\rm new}(\vct x) \neq y$, being $y$ their true class. These samples are referred to as \textit{negative flips} (NFs), as they are responsible for worsening the performance of the new model when compared to the old one. Conversely, the samples that were misclassified by the old model but not by the new one are referred to as \textit{positive flips} as they improve the overall accuracy of \fnew on $\set T$.
More formally, we can define the NF rate (measured on the test set $\set T$) as suggested by Yan et al.~\cite{yan2021positive}:
\begin{equation}
    {\rm NF (\%)} = \frac{1}{m} \sum_{j=1}^m \mathbb I \left (\hat y_j^{\rm old} = y_j \land \hat y_j^{\rm new} \neq y_j \right) \, ,
\label{eq:nf}
\end{equation}
where $\mathbb I$ is the indicator function, which equals one only if the input statement holds true (and zero otherwise), $\land$ is the \textit{logical and} operator, and we make use of the compact notation $\hat y_j^{\rm old}$ and $\hat y_j^{\rm new}$ to denote respectively $\hat y^{\rm old}(\vct x_j)$ and $\hat y^{\rm new}(\vct x_j)$. The regression of accuracy induced by NFs after model update is also exemplified in Fig.~\ref{fig:types_nf} (\textit{left}).


\subsubsection{Positive-Congruent Training (PCT)}
\label{sect:regression.pct}

To reduce the presence of NFs after model update, and prevent the system users from experiencing a regression of accuracy when using the new model, Yan et al.~\cite{yan2021positive} proposed a technique named \textit{Positive-Congruent Training} (PCT).
The underlying idea of PCT is to add a knowledge-distillation term to the loss function optimized when training the new model \fnew. 
This term is referred to as \textit{focal distillation}, as it forces \fnew to produce outputs similar to those of \fold on the samples that the old model correctly classified. 
The PCT approach can be formally described as:
\begin{align}
\fnew \in \argmin_{f \in \mathcal{F}} \sum_{i=1}^{ n}  L(y_i, f(\vct x_i))  + \lambda \cdot L_{FD}(f(\vct x_i), \fold(\vct x_i)) \, ,
\label{eq:pct}
\end{align}
being $L$ a standard loss function (\eg the cross-entropy loss), $L_{FD}$  the focal-distillation loss, and $\lambda$ a trade-off hyperparameter. The focal-distillation loss takes the logits $f(\vct x_i)$ of the new model \f being optimized as inputs, along with the logits $\fold(\vct x_i)$ provided by the old model, and it is computed as:
\begin{align}
L_{FD}= \left (\alpha + \beta \cdot \mathbb I(\hat{y}^{\rm old}_i=y_i) \right) \cdot L_D(f(\vct x_i), \fold(\vct x_i)) \, ,
\label{eq:focal-loss}
\end{align}
where $\alpha$ and $\beta$ are two (non-negative) hyperparameters, and $L_D$ is the distillation loss. The hyperparameter $\alpha$ forces the distillation of the old model over the whole sample set, while $\beta$ is used to upweight the contribution of the samples that were correctly classified by \fold. In this manner, the new model tends to mimic the behavior of the old model where the latter performed correctly, reducing the potential mistakes induced by NFs. While in their work Yan et al.~\cite{yan2021positive} experimented with different distillation losses, we focus here on the most promising one, named \textit{focal distillation with logit matching} (FD-LM). This distillation loss simply measures the squared Euclidean distance between the logits of the model \f being optimized and those of the old model \fold:
\begin{equation}
    L_D(f(\vct x_i), \fold(\vct x_i)) = \frac{1}{2}   \| \f(\vct x_i) - \fold(\vct x_i) \|_2^2 .
\label{eq:distill}
\end{equation}

While PCT has empirically proven to reduce NFs, it has not been designed to deal with the regression of robustness.
Furthermore, its theoretical properties have not been analyzed, and it thus remains unclear whether it provides a sound, statistically-consistent estimator.

\subsection{Regression of Robustness: Robustness Negative Flips (RNFs)}
\label{sect:regression.rnfs}

After introducing the notion of regression of accuracy, induced by the presence of NFs after model update, we argue here that incremental changes to machine-learning models can also affect their security against \emph{adversarial examples}~\cite{biggio2013evasion,szegedy2014intriguing}, \ie carefully-crafted input perturbations aimed to cause misclassifications at test time.
More formally, adversarial examples are found by solving the following optimization:
\begin{equation}
\max_{\vct x^\prime \in \mathcal B(\vct x)} L(y, f(\vct x^\prime))  \, ,
\label{eq:advx}
\end{equation}
where $\mathcal{B}(\vct x) = \{\vct x^\prime \in \set X  : \| \vct x^\prime - \vct x\|_p \leq \varepsilon \}$ defines the perturbation domain,  $\|\cdot \|_p$ a suitable $\ell_p$ norm, and $\varepsilon$ the perturbation budget. In practice, the goal is to find a sample $\vct x^\prime \in \mathcal{B}(\vct x)$ that is misclassified by \f, \ie for which $\hat y(\vct x^\prime) \neq y$.

As shown in Fig.~\ref{fig:types_nf} (\textit{right}), the adversarial example $\vct x^\prime$ is optimized within an $\ell_\infty$-norm (box) constraint of radius~$\varepsilon$ centered on the source sample $\vct x$.
It is not difficult to see that, while no adversarial example can be found against the old model \fold (as the box constraint never intersects its decision boundary), the same does not hold for the new model \fnew, which is evaded by the adversarial example $\vct x^\prime$. Accordingly, when updating the old model \fold with \fnew, even if the overall robust accuracy may increase, \fnew might be evaded by some adversarial examples that were not found against \fold, causing what we call \emph{robustness negative flips} (RNFs).
More formally, we define the RNF rate as:
\begin{equation}
    {\rm RNF (\%)} = \frac{1}{m} \sum_{j=1}^m \mathbb I \left (\hat y^{\rm old}(\vct x_j^\prime) = y_j \land \hat y^{\rm new}(\vct x_j^{\prime \prime}) \neq y_j \right) \, ,
\label{eq:rnf}
\end{equation}
where $\vct x_j^{\prime}$ and $\vct x_j^{\prime \prime}$ are the adversarial examples obtained by solving Problem~\eqref{eq:advx} against \fold and \fnew, respectively. This means that adversarial examples are re-optimized against each model, and the statement $\hat y^{\rm old}(\vct x_j^\prime) = y_j$ holds true only if no adversarial example within the given perturbation domain $\set B(\vct x)$ is found against \fold,  \ie no sample in $\set B(\vct x)$ is misclassified by \fold, whereas it suffices to find one evasive sample against \fnew to conclude that $\hat y^{\rm new}(\vct x_j^{\prime \prime}) \neq y_j$. This makes measuring regression of robustness more complex, as it is defined over a perturbation domain $\set B(\vct x)$ around each input sample $\vct x$, rather than just on the set of input samples. However, as we will see in our experiments, it can be reliably estimated using state-of-the-art attack algorithms and best practices to optimize adversarial examples~\cite{carlini19-arxiv,pintor2021fast,croce2020reliable}.

To conclude, note that an input sample $\vct x$ may be correctly classified by \fold, and at the same time no corresponding adversarial example may exist, \ie no $\vct x^\prime$ evading \fold can be found within the given perturbation domain $\set B(\vct x)$, as shown in the right plot of Fig.~\ref{fig:types_nf}. After the update, it may happen that \fnew misclassifies the input sample $\vct x$. Thus, by definition, $\vct x$ coincides with the solution of Problem~\eqref{eq:advx}, \ie its adversarial example $\vct x^\prime$. This case will be accounted for as an NF and an RNF simultaneously. However, as discussed in our experiments, these \textit{joint} negative flips are typically very rare and overall negligible.


\section{Secure Model Updates via Robustness-Congruent Adversarial Training}
\label{sect:rct}


We present here our \textit{Robustness-Congruent Adversarial Training} (RCAT) approach to updating machine-learning models while keeping a low number of \anfs and \rnfs (Sect.~\ref{sect:rcat}).
We then demonstrate how RCAT, as well as PCT, provide statistically-consistent estimators in Sect.~\ref{subsec:Consistency}.

\subsection{Robustness-Congruent Adversarial Training} \label{sect:rcat}

We formulate RCAT as an extension of Problem~\eqref{eq:pct} that includes \textit{adversarial training}, \ie the optimization of adversarial examples during model training. This is a well-known practice used to improve adversarial robustness of machine-learning models~\cite{madry2017towards}. Before introducing RCAT, we propose a trivial extension of PCT with adversarial training, which we refer to as \textit{Positive-Congruent Adversarial Training} (PCAT).

\myparagraph{Positive-Congruent Adversarial Training (PCAT).} PCAT amounts to solving the following problem:
\begin{align}
\min_{f \in \mathcal{F}} \sum_{i=1}^{ n} \max_{\vct x_i^\prime \in \set B_i}  L(y_i, f(\vct x_i^\prime))  + \lambda \cdot L_{FD}(f(\vct x_i^\prime), \fold(\vct x_i^\prime)) \, ,
\label{eq:pct-at}
\end{align}
where $\set B_i$ is used to compactly denote $\set B(\vct x_i)$. This is a min-max optimization problem that aims to find the worst-case adversarial example $\vct x_i^\prime$ for each training sample $\vct x_i$, while upweighting the distillation loss on the adversarial examples that were not able to fool \fold.
The underlying idea is to preserve robustness on the samples for which no adversarial example was found against \fold, thereby reducing RNFs.
However, as we will show in our experiments, this technique is not very effective in preventing regression of robustness, even though it provides a reasonable baseline for comparison. In particular, the main problems that arise when trying to use PCAT are: (i) it is not easy to define an effective hyperparameter tuning strategy; and (ii) it is not possible to exploit an already-trained, updated, and more robust model directly.

\myparagraph{Robustness-Congruent Adversarial Training (RCAT).} With respect to the baseline idea of PCAT, we  define RCAT as:
\begin{align}
\min_{f \in \mathcal F} \sum_{i=1}^n \max_{\vct x^\prime_i \in \mathcal B_i} \gamma \cdot L(y_i, f(\vct x^\prime_i)) +  \alpha \cdot L_D(f({\vct x}_i^\prime), \fsrc({\vct x}^\prime_i)) \, + \nonumber \\  \beta \cdot \mathbb I(\hat{y}^{\rm old}(\vct x^\prime_i)=y_i) \cdot L_D(f({\vct x}^\prime_i), \fold({\vct x}^\prime_i)) \, .
\label{eq:rfat}
\end{align} 
This formulation presents two main changes with respect to PCAT (Problem~\ref{eq:pct-at}), to overcome the two aforementioned limitations of such method.
First, to facilitate hyperparameter tuning,  we redefine the range of the hyperparameters $\alpha,\beta \in [0,1]$, while fixing $\gamma=1-\alpha-\beta$, so that the three hyperparameters sum up to 1. We also remove the hyperparameter $\lambda$ as it is redundant.
Second, we use \fsrc instead of \fold in the $\alpha$-scaled term of the focal distillation loss (Eq.~\ref{eq:focal-loss}).
The reason is that one may want to update a model \fold with an already-trained \textit{source} model \fsrc that exhibits improved accuracy and robustness, while also reducing NFs and RNFs after update.
As \fsrc can be used to initialize \f before training via RCAT, it is not reasonable to enforce the behavior of \fold over the whole input space. We can indeed try to preserve the behavior of the improved \fsrc model over the whole input space, while enforcing that of \fold only on those regions of the input space in which no adversarial examples against \fold are found. This is especially convenient when dealing with robust models, as they are usually trained with complicated adversarial training and data augmentation variants, resulting in a significant increase in computational complexity. Thus, if an already-trained, more robust model becomes available, it can be readily used in RCAT as \fsrc, as well as to initialize \f before optimizing it, while \fold can be used to reduce NFs and RNFs in the $\beta$-scaled term.

To summarize, with respect to the baseline PCT formulation in Eq.~\eqref{eq:pct}~\cite{yan2021positive}, RCAT provides the following modifications:
\begin{enumerate}
    \item it includes an adversarial training loop to reduce RNFs;
    \item it redefines the hyperparameters to facilitate tuning; and
    \item it allows distilling from a different model than \fold over the whole input space, when available, to retain better accuracy and robustness.
\end{enumerate}

\begin{algorithm}[t!]
    \SetKwInOut{Input}{Input}
    \SetKwInOut{Output}{Output}
    \SetKwComment{Comment}{$\triangleright$\ }{}
    \DontPrintSemicolon
    \Input{
    $\trainset$, the training dataset; $\set L$, the loss defined in \autoref{eq:rfat}, with its fixed hyperparameters $\alpha$, and $ \beta$; \fsrc, the source/init model; \fold, the old model;
     $a$, the attack algorithm used to solve the inner maximization in Eq.~\eqref{eq:rfat} over
    $\mathcal{B}$, the perturbation domain; $\eta$, the learning rate; and $E$, the number of epochs;}
    \Output{$\fnew$, the   model trained with RCAT.}
    
    $\f_{\vct w} \leftarrow \fsrc$ \Comment*[r]{Initialize the model as \fsrc}\label{line:init}
    
    \For{$i \in [1, E]$}{ \label{line:iter_epochs}
        \For{$(\vct x, y) \in \trainset$}{ \label{line:iter_samples}
            $\vct x^\star \leftarrow \argmax_{\vct x^\prime \in \set B(\vct x)} \set L(y, f_{\vct w}(\vct x^\prime))$ \Comment*[r]{Compute the adversarial example $\vct x^\star$ with attack $a$.}\label{line:adv_attack}
                
            $\vct w \leftarrow \vct w - \eta \nabla_{\vct w} \set L(y, f_{\vct w}(\vct x^\star))$ \Comment*[r]{Update the model parameters $\vct w$.} \label{line:update_pars}
            }
        }
        \textbf{return} $\fnew \leftarrow \f_{\vct w}$ \Comment*[r]{Return the RCAT model.}\label{line:return}
\caption{The \rfat algorithm.}
\label{algo:rcat}
\end{algorithm}

\subsubsection{Solution Algorithm} 

We describe here the gradient-based approach used to solve the RCAT learning problem defined in Eq.~\eqref{eq:rfat}. A similar algorithm can be used to solve the PCAT problem in Eq.~\eqref{eq:pct-at}. Let us assume that the function $f$ is parameterized by $\vct w$. Then, the gradient of the objective function in Eq.~\eqref{eq:rfat} with respect to the model parameters $\vct w$ is given as:
\begin{align}
    \vct w \leftarrow \vct w - \eta \sum_{(\vct x, y) \in \trainset} \nabla_{\vct w} \max_{\vct x^\prime \in \set B(\vct x)} \set L(y, f_{\vct w}(\vct x^\prime)) \, ,
\end{align}
where we make the dependency of $f$ on $\vct w$ explicit as $f_{\vct w}$, and use the symbol $\set L$ to denote the sample-wise loss defined in Eq.~\eqref{eq:rfat}, which implicitly depends on $\fsrc$ and $\fold$.

According to Danskin’s theorem~\cite{madry2017towards}, the gradient of the inner maximization is equivalent to the gradient of the inner objective computed at its maximum. This means that, if we assume $\vct x^\star \in \argmax_{\vct x^\prime \in \set B (\vct x)} L (y, f(\vct x^\prime))$, the gradient update formula can be rewritten as:
\begin{align}
    \vct w \leftarrow \vct w - \eta \sum_{(\vct x, y) \in \trainset} \nabla_{\vct w} L(y, f_{\vct w}(\vct x^\star)) \, .
\end{align}
However, computing the exact solution of the inner maximization might be too computationally demanding.
Madry et al.~\cite{madry2017towards} have nevertheless shown that it is still possible to use an approximate solution with good empirical results, by relying on an adversarial attack that computes a perturbation close enough to that of the exact solution.

Under these premises, we can finally state the RCAT algorithm used to optimize Eq.~\eqref{eq:rfat}, given as \autoref{algo:rcat}.
%
%
The algorithm starts by initializing the model $\f_{\vct w}$ with $\fsrc$ (\autoref{line:init}).
It then runs for $E$ epochs (\autoref{line:iter_epochs}), looping over the whole training samples in each epoch (\autoref{line:iter_samples}). 
In each iteration, RCAT runs the attack algorithm $a$ to optimize the adversarial example $\vct x^\star$ within the feasible domain $\set B(\vct x)$ (\autoref{line:adv_attack}).
Then, it uses the adversarial example $\vct x^\star$ to update the model parameters $\vct w$ along the gradient direction (\autoref{line:update_pars}).
While we present here a sample-wise version of the RCAT algorithm, it is worth remarking that it is straightforward to implement it batch-wise.

\subsection{Consistency Results}
\label{subsec:Consistency}

We now analyze the formal guarantees that can be derived for the learning problem when including a \textit{non-regression constraint} to reduce regression of accuracy and robustness.
Under the assumption that samples are independent and identically distributed (i.i.d.) and in the absence of constraints, it is well known that learning algorithms that optimize either the accuracy or robustness yield \textit{consistent} statistical estimators, \ie they converge to the correct model as the number of samples $n$ increases, with a rate of $O(\sfrac{1}{\sqrt{n}})$ in the general case~\cite{hastie2009elements, OnetoJ072}.
This means that collecting samples to increase the training set is generally worthwhile, as it is expected to improve the performance of the model.
However, neither consistency nor the convergence rate is guaranteed when the optimization problem includes constraints~\cite{hastie2009elements}.
In this section, we are the first to show that including the non-regression constraint inside the optimization problem produces an estimator that is consistent on both accuracy and robustness, while also preserving the same convergence rate.

\subsubsection{Preliminaries}

Let us start by formally defining the necessary terms that we will use in this section. We slightly change the notation here to improve readability.
We will denote point-wise loss functions as $\ell$, and dataset (distribution) loss functions as $\mathsf{L}$. We also redefine $f$ such that it outputs the predicted label rather than the logits, \ie $f : \set X \mapsto \set Y$; and denote the indicator function with $\mathbb I(\cdot)$, which returns one if its argument is true, and zero otherwise.
The goal is to design a learning algorithm that chooses a model $f \in \mathcal{F}$ to approximate the posterior probability $\mathbb{P}\{y|\boldsymbol{x}\}$ according to a loss function $\ell(f,\boldsymbol{z}) \in [0,\infty)$, such that $\ell(f,\boldsymbol{z}) = 0$ if the prediction is correct, \ie $y = f(\boldsymbol{x})$, with $\boldsymbol{z} = (\boldsymbol{x}, y)$. This algorithm is often defined via \textit{empirical risk minimization}:
\begin{align}\label{eq:erm}
\min_{f \in \mathcal{F}} \hat{\mathsf{L}}(f,\mathcal{D}) \, ,
\end{align}
where $\hat{\mathsf{L}}(f,\mathcal{D}) = \sfrac{1}{n} \sum_{\boldsymbol{z} \in \mathcal{D}} \ell(f,\boldsymbol{z})$ is the empirical risk estimated from the training data $\set D$.
The hypothesis space $\mathcal{F}$ may be explicitly defined by means of the functional form of $f$ (e.g, linear, convolutions, transformers), or via regularization (\eg $L_p$ norms)~\cite{goodfellow2016deep,aggarwal2018neural}. It may also be implicitly defined via optimization (\eg stochastic gradient descent, early stopping,  dropout)~\cite{goodfellow2016deep,srivastava2014dropout}.
Problem~\eqref{eq:erm} is the empirical counterpart of the \textit{risk minimization} problem:
\begin{align}\label{eq:rm}
\min_{f \in \mathcal{F}} \mathsf{L}(f,\mathcal{Z}) \, ,
\end{align}
where $\mathsf{L}(f,\mathcal{Z}) = \mathbb{E}_{\boldsymbol{z}: \boldsymbol{z} \in \mathcal{Z}} \{ \ell(f,\boldsymbol{z}) \}$ is the \textit{true} risk, \ie the risk computed over the whole distribution of samples $\set Z$.

When considering adversarial robustness, one can define a sample-wise \textit{robust} loss $\tilde{\ell}(f,\boldsymbol{z})$ that quantifies the error of $f$ on the adversarial examples $\tilde{\boldsymbol{x}}$~\cite{bartlett2002rademacher,yin2019rademacher} as:
\begin{align}
\label{eq:advloss}
\tilde{\ell}(f,\boldsymbol{z}) = \max_{\tilde{\boldsymbol{x}} \in \mathcal{B}(\boldsymbol{x})} \ell(f,(\tilde{\boldsymbol{x}},y)) \, .
\end{align}
By using this robust loss as the sample-wise loss of Problems~\eqref{eq:erm} and~\eqref{eq:rm}, we are able to define both the empirical and deterministic optimization problems that amount to maximizing adversarial robustness via \emph{adversarial training}~\cite{madry2017towards}.

We can now introduce the non-regression constraint into the learning problem. 
Under the assumption that the training set $\mathcal{D}$ and the hypothesis space $\mathcal{F}$ are independent from the old model $\fold$,\footnote{We could remove this hypothesis using \eg~\cite{bassily2021algorithmic,pmlr-v51-russo16}, but this would simply over-complicate the presentation with no additional contribution.} we can rewrite Problems~\eqref{eq:erm} and~\eqref{eq:rm} as:
\begin{align}
& \hat{f} = \arg\min_{f \in \mathcal{F}} \hat{\mathsf{L}}(f,\mathcal{D}),\quad \text{s.t.}\ \hat{\mathsf{L}}(f,\mathcal{D}_0) \leq \hat{\epsilon}, \label{eq:erm_nr} \\
& f^* = \arg\min_{f \in \mathcal{F}} \mathsf{L}(f,\mathcal{Z}),\quad \text{s.t.}\ \mathsf{L}(f,\mathcal{Z}_0) \leq \epsilon, \label{eq:rm_nr}
\end{align}
where $\hat f$ is the empirical approximation of the correct function~$f^*$, $\mathcal{D}_0 = \{\boldsymbol{z}: \boldsymbol{z} \in \mathcal{D}, \ell(\fold, \boldsymbol{z}) = 0\}$ is the set of training samples that are correctly classified by the old model, and $\mathcal{Z}_0 = \{\boldsymbol{z}: \boldsymbol{z} \in \mathcal{Z}, \ell(\fold,\boldsymbol{z}) = 0 \}$ is the set of samples correctly classified by $\fold$ over the whole distribution.

Problems~\eqref{eq:erm_nr}-\eqref{eq:rm_nr} can represent both the problem of reducing negative flips of accuracy (\anfs), and that of reducing robustness negative flips (\rnfs). In particular, if we use the zero-one loss $\ell(f,\boldsymbol{z})=\mathbb I(f(\vct x) \neq y)$ as the sample-wise loss, the constraints in Problems~\eqref{eq:erm_nr}-\eqref{eq:rm_nr}
will enforce solutions with low regression of accuracy (\ie \anf rates). 
Instead, when using the robust zero-one loss $\tilde{\ell}(f,\boldsymbol{z}) = \max_{\tilde{\boldsymbol{x}} \in \mathcal{B}(\boldsymbol{x})} \mathbb{I}(f(\tilde{\boldsymbol{x}}) \neq y)$, such constraints will enforce solutions that exhibit low regression of robustness (\ie reduce \rnfs). 

Let us conclude this section by discussing the role of $\hat \epsilon$ and $\epsilon$ in the aforementioned constraints.
While the desired $\epsilon \in [0,\infty)$ should be set to zero to have zero regression (\ie no negative flips of accuracy or robustness),  $\hat{\epsilon}\in [0,\infty)$ is usually set to a small value $\hat{\epsilon} > 0$ for two main reasons.
Theoretically, as discussed in the following, this is a sufficient condition to ensure Problem~\eqref{eq:erm_nr} to be consistent with respect to Problem~\eqref{eq:rm_nr}.
Practically, as shown in~\autoref{sect:exp}, setting $\hat{\epsilon} > 0$ allows us to obtain larger improvements in the test error of $\hat{f}$ since $\mathcal{F}$ may be not perfectly designed and the number of (noisy) samples is limited.

\subsubsection{Main Result}
\label{sect:main-result}

We prove here that the estimator $\hat{f}$ defined in Eq.~\eqref{eq:erm_nr} is \textit{consistent} with the function $f^\star$ that would be learned over the whole distribution (Eq.~\ref{eq:rm_nr}).
To this end, we assume that the following relationship holds, with probability at least $(1-\delta)$:
\begin{align}
\label{eq:eq_leanable}
\max_{f \in \mathcal{F}} \left| \hat{\mathsf{L}}(f,\mathcal{D}) - \mathsf{L}(f,\mathcal{Z}) \right| \leq \mathsf{B}(\delta,n,\mathcal{F}) \xrightarrow{O(\sfrac{1}{\sqrt{n}})} 0,
\end{align}
where $B(\delta,n,\mathcal{F})$ goes to zero as $n \rightarrow \infty$ if the hypothesis space $\mathcal{F}$ is learnable, in the classical sense, with respect to the loss~\cite{shalev2014understanding}.
If this holds, then the hypothesis space $\mathcal{F}$ is also learnable when dealing with the adversarial setting, \ie when using the sample-wise loss defined in Eq.~\eqref{eq:advloss}. This can be proved using the Rademacher complexity, as done in~\cite{yin2019rademacher,bartlett2002rademacher,OnetoJ072}.
Note also that, in the general case, $B(\delta,n,\mathcal{F})$ goes to zero as $O(\sfrac{1}{\sqrt{n}})$~\cite{shalev2014understanding}.

Under this assumption, we prove that $\hat{f}$ is consistent in the following sense.
For a particular value of $\hat{\epsilon}$, we show that
\begin{align}
\label{eq:consistency}
\mathsf{L}(\hat{f},\mathcal{Z}) - \mathsf{L}(f^*,\mathcal{Z}) \xrightarrow{O(\sfrac{1}{\sqrt{n}})} 0, \quad
\mathsf{L}(\hat{f}, \mathcal{Z}_0) \xrightarrow{O(\sfrac{1}{\sqrt{n_0}})} \epsilon,
\end{align}
where $n_0 = | \mathcal{D}_0 |$ is the number of samples that were correctly predicted by the old model \fold, and then if $n \rightarrow \infty$ we also have $n_0 \rightarrow \infty$.
This means that, if the hypothesis space $\set F$ is learnable and the empirical risk minimizer is consistent in the classical setting, then it is also consistent when we add the non-regression constraint to the learning problem, both in the case of \anfs and \rnfs.
\begin{theorem}
\label{thm:mainthm}
Let us consider a learnable $\mathcal{F}$, in the sense of Eq.~\eqref{eq:eq_leanable}, and $\hat{f}$ and $f^*$ defined as in Eqns.~\eqref{eq:erm_nr} and~\eqref{eq:rm_nr} respectively.
Then it is possible to prove the result of 
Eq.~\eqref{eq:consistency}.
\end{theorem}
\begin{proof}
Let us note that, thanks to Eq.~\eqref{eq:eq_leanable}, with probability at least $(1-\delta)$ it holds that
\begin{align}
\max_{f \in \mathcal{F}} \left| \mathsf{L}(f, \mathcal{Z}_0) - \hat{\mathsf{L}}(f,\mathcal{D}_0) \right| \leq \mathsf{B}(\delta,n_0,\mathcal{F}) \, .
\end{align} 
We can then state that, with probability at least $(1-\delta)$,
\begin{align}
\label{eq:prop1}
& \{f:
f \in \mathcal{F},
\mathsf{L}(f, \mathcal{Z}_0) \leq \epsilon,
\} \nonumber \\
& \subseteq
\{f:
f \in \mathcal{F},
\hat{\mathsf{L}}(f,\mathcal{D}_0)
\leq \hat{\epsilon} = \epsilon + \mathsf{B}(\delta,n_0,\mathcal{F}),
\}
\subseteq
\mathcal{F}.
\end{align}
Thanks to Eqns.~\eqref{eq:eq_leanable} and~\eqref{eq:prop1}, we can thus decompose the excess risk, with probability at least $(1-3\delta)$, as:
\begin{align}
\mathsf{L}(\hat{f},\mathcal{Z})
- \mathsf{L}(f^*,\mathcal{Z}) 
& =
\mathsf{L}(\hat{f},\mathcal{Z})
- \hat{\mathsf{L}}(\hat{f},\mathcal{D})
+ \hat{\mathsf{L}}(\hat{f},\mathcal{D}) \nonumber\\
& \quad - \hat{\mathsf{L}}(f^*,\mathcal{D})
+ \hat{\mathsf{L}}(f^*,\mathcal{D})
- \mathsf{L}(f^*,\mathcal{Z}) \nonumber \\
& \leq
\hat{\mathsf{L}}(\hat{f},\mathcal{D})
- \hat{\mathsf{L}}(f^*,\mathcal{D})
+ 2 \mathsf{B}(\delta,n,\mathcal{F}) \nonumber \\
& \leq
2 \mathsf{B}(\delta,n,\mathcal{F}),
\end{align}
which proves the first statement in  Eq.~\eqref{eq:consistency}.
Furthermore, thanks to Eqns.~\eqref{eq:eq_leanable} and~\eqref{eq:prop1}, with probability at least $(1-\delta)$, it also holds that
\begin{align}
\mathsf{L}(\hat{f}, \mathcal{Z}_0)
& \leq
\mathsf{L}(\hat{f}, \mathcal{Z}_0) 
- \hat{\mathsf{L}}(\hat{f},\mathcal{D}_0)
+ \hat{\mathsf{L}}(\hat{f},\mathcal{D}_0) \nonumber \\
& \leq
\mathsf{B}(\delta,n_0,\mathcal{F})
+ \hat{\mathsf{L}}(\hat{f},\mathcal{D}_0) \nonumber \\
& \leq
\epsilon + 2 \mathsf{B}(\delta,n_0,\mathcal{F}),
\end{align}
which proves the second statement in Eq.~\eqref{eq:consistency}.
\end{proof}
\subsubsection{Consistency of PCT, PCAT, and RCAT}
\label{subsec:Practice}
The previous section proves the consistency of $\hat{f}$ with the usual convergence rate when the learning problem includes a non-regression constraint. This means that designing updates that reduce \anfs or \rnfs does not compromise any of the standard properties that normally hold for learning algorithms.

We show here that the constrained learning problem defined in the previous section can be rewritten using a penalty term instead of requiring an explicit non-regression constraint, and how this maps to the formulation of PCT (Eq.~\ref{eq:pct}), PCAT (Eq.~\ref{eq:pct-at}) and RCAT (Eq.~\ref{eq:rfat}), to show that they all provide consistent estimators with the usual convergence rate.  
As shown in~\cite{ito2014inverse,oneto2016tikhonov}, Problem~\eqref{eq:erm_nr} is equivalent to:
\begin{align}
\label{eq:final_alg}
\min_{f \in \mathcal{F}} \hat{\mathsf{L}}(f,\mathcal{D}) + \mu \, \hat{\mathsf{L}}(f,\mathcal{D}_0) \, ,
\end{align}
with $\mu \in [0, \infty)$. It is now straightforward to map this formulation to that of PCT, PCAT, and RCAT. The underlying idea is to group the loss terms that are computed over the whole training set $\set D$ in the formulations of PCT, PCAT, and RCAT within the first term of Eq.~\eqref{eq:final_alg}, while assigning the loss term computed on the subset $\set D_0$ to the $\mu$-scaled term.

\myparagraph{PCT.} For PCT (Eq.~\ref{eq:pct}), we can set:
\begin{flalign}
 & \hat{\mathsf{L}}(f,\mathcal{D}) = \sum_{(\vct x,y) \in \set D} L(y, f(\vct x)) + \lambda \alpha L_D(f(\vct x), \fold(\vct x)) , \\
 &\hat{\mathsf{L}}(f,\mathcal{D}_0) = \sum_{(\vct x, y) \in \set D_0} L_D(f(\vct x), \fold(\vct x)) , 
\end{flalign}
and $\mu= \lambda \beta$. The set $\set D_0$ contains the samples that are classified correctly by \fold, \ie for which $\hat{y}^{\rm old}(\vct x)=y$. Recall also that here the sample-wise loss $L$ is the standard (non-robust) cross-entropy loss.

\myparagraph{PCAT.} For PCAT (Eq.~\ref{eq:pct-at}), we can set $\mu=\lambda \beta$, and use  sample-wise \textit{robust} losses instead of  the standard ones used in PCT:
\begin{flalign}
 & \hat{\mathsf{L}}(f,\mathcal{D}) = \sum_{(\vct x, y) \in \set D} \max_{\vct x^\prime \in \set B_{\vct x}} L(y, f(\vct x^\prime)) + \lambda \alpha L_D(f(\vct x^\prime), \fold(\vct x^\prime)) , \\
  & \hat{\mathsf{L}}(f,\mathcal{D}_0) = \sum_{(\vct x, y) \in \set D_0} \max_{\vct x^\prime \in \set B_{\vct x}}  L_D(f(\vct x^\prime), \fold(\vct x^\prime)),
\end{flalign}
where $B_{\vct x}$ compactly denotes $\set B(\vct x)$, and $\set D_0$ contains the adversarial examples that are correctly classified by \fold, \ie the training samples for which $\hat{y}^{\rm old}(\vct x^\prime)=y$, $\forall \vct x^\prime \in \set B(\vct x)$.

\myparagraph{RCAT.} For RCAT (Eq.~\ref{eq:rfat}), we can set:
\begin{flalign}
 &\hat{\mathsf{L}}(f,\mathcal{D}) = \sum_{(\vct x, y) \in \set D} \max_{\vct x^\prime \in \set B_{\vct x}} \gamma L(y, f(\vct x^\prime)) +  \alpha L_D(f(\vct x^\prime), \fsrc(\vct x^\prime)) , \\
 &\hat{\mathsf{L}}(f,\mathcal{D}_0) = \sum_{(\vct x, y) \in \set D_0} \max_{\vct x^\prime \in \set B_{\vct x}} L_D(f(\vct x^\prime), \fold(\vct x^\prime)),
\end{flalign}
and $\mu = \beta$. The set $\set D_0$ is defined as in the case of PCAT.

It is worth remarking that not only PCAT and RCAT provide consistent statistical estimators, but this also applies to PCT~\cite{yan2021positive} and any algorithm derived from a formulation that includes a non-regression constraint or penalty term.

\section{Experimental Analysis}
\label{sect:exp}
To show the validity of the \rcat method, we report here an extensive experimental analysis involving several robust machine-learning models designed for image classification.
After describing the experimental setup (\autoref{subsec:exp_setup}), we demonstrate that the problem of regression of robustness is relevant when replacing a robust machine-learning model with an improved state-of-the-art model (\autoref{subsec:existence_rnf}). 
We then show that \rfat can better mitigate the regression compared to PCT, PCAT, and na\"ive model update strategies, \edit{not only when considering a single update} (\autoref{sect:ftuning}), 
\edit{but also in the case of multiple sequential updates (\autoref{sect:sequential}).}

\subsection{Experimental Setup}
\label{subsec:exp_setup}
Here, we detail the experimental setup used in our analyses.

\myparagraph{Datasets.} 
\edit{We consider CIFAR-10 and ImageNet as they have been extensively used to benchmark the adversarial robustness of image classifiers in the well-known RobustBench framework~\cite{croce2021robustbench}.}
CIFAR-10 consists of 60,000 color images of size $32 \times 32$ belonging to $10$ different classes, including 50,000 training images and 10,000 test images.
Following current evaluation standards~\cite{croce2021robustbench}, we use 40,000 training images to update models with PCT, PCAT, and RCAT, the remaining 10,000 training images as a validation set for hyperparameter tuning, and a subset of 2,000 test images to evaluate the performance metrics (as detailed below).
\edit{For ImageNet, we consider only its validation set (\ie 50,000 color images of size $224 \times 224$ divided into 1,000 classes), using 36,000 images to update the models, 9,000 as a validation set for hyperparameter tuning, and the remaining 5,000 for testing.}

\myparagraph{Robust Models.} 
\edit{In our experiments, we consider robust models from RobustBench~\cite{croce2021robustbench}, evaluated against $\ell_{\infty}$-norm attacks with a perturbation budget of $\epsilon = 8/255$ and $\epsilon = 4/255$ for CIFAR-10 and ImageNet, respectively. }
\edit{In particular,} we consider seven CIFAR-10 robust models denoted from the least to the most robust with \edit{$C_1, \ldots, C_7$}, and originally proposed respectively in~\cite{Engstrom2019Robustness, Zhang2020Attacks, Rice2020Overfitting, Rade2021Helper, Hendrycks2019Using, Addepalli2021Towards, Carmon2019Unlabeled}. 
\edit{We then consider five ImageNet robust models, denoted from the least to the most robust with $I_1, \ldots, I_5$, and originally proposed in~\cite{Salman2020Do_R18, Engstrom2019Robustness, Chen2024Data_WRN_50_2} ($I_1, I_2, I_3$), and~\cite{Liu2023Comprehensive} ($I_4, I_5$).}
\edit{All the considered models} use different architectures, training strategies, and exhibit different trade-offs between accuracy and robustness, enabling us to simulate different scenarios in which there would be a clear incentive in replacing an old model with an improved one.

\myparagraph{Model Updates.} We consider four different model update strategies: (i) na\"ive, where we just replace the old model with an already-trained, more robust one; (ii) PCT; (iii) PCAT; and (iv) RCAT.
When using PCT, PCAT, and RCAT, we fine-tune \edit{CIFAR-10} models using \textit{E} = \numepochs training epochs, and batches of \batchsize samples. 
\edit{For ImageNet models, we use 10 epochs and batches of 128 samples for the first four models. 
We use batches of 64 samples for the last model instead, to fit in memory, and reduce the number of epochs to 5 to keep the number of iterations unchanged.}
We set the learning rate $\eta =$\lr \edit{for training on both datasets}.
PCAT and RCAT also require implementing adversarial training using an attack algorithm $a$, as detailed in \autoref{algo:rcat}. To this end, we follow the implementation of the \textit{Fast Adversarial Training} (FAT) approach proposed in~\cite{wong2020fast}, which has been shown to provide similar results to the computationally-demanding adversarial training (AT) proposed in~\cite{madry2017towards}, while being far more efficient. The underlying idea of FAT is to randomly perturb the initial training samples and then use a fast, non-iterative attack to compute the corresponding adversarial perturbations. In particular, instead of using the iterative Projected Gradient Descent (PGD) attack as done in AT~\cite{madry2017towards}, FAT uses the so-called Fast Gradient Sign Method (FGSM)~\cite{goodfellow2014explaining}, \ie a much faster, non-iterative $\ell_\infty$-norm attack. Even if FGSM is typically less effective than PGD in finding adversarial examples, when combined with random initialization to implement adversarial training, it turns out to achieve competing results~\cite{wong2020fast}.

\myparagraph{Hyperparameter Tuning.}
We choose the hyperparameters of \pct, \pctat, and \rfat that minimize the overall number of negative flips (NFs) and robustness negative flips (RNFs) \edit{(\ie their sum)} on the validation set, to achieve a reasonable trade-off between reducing the regression of accuracy and that of robustness with respect to the na\"ive strategy of replacing the previous model with the more recent one.
For \pct (\autoref{eq:pct}) and \pctat (\autoref{eq:pct-at}) baseline methods, we fix $\lambda = 1$ and $\alpha=1$, and we run a grid-search on $\beta \in \{1, 2, 5, 10\}$, as recommended by Yan et al.~\cite{yan2021positive}.
For \rfat, we run a grid-search on $(\alpha, \beta) \in \{(0.75, 0.2), (0.7, 0.2), (0.5, 0.4), (0.3, 0.6)\}$ while fixing $\gamma = 1 - \alpha - \beta$.
These configurations attempt to give more or less importance to the non-regression penalty term, while RCAT also re-balances the other loss components.

\begin{figure*}
    \centering
        \begin{subfigure}{0.49\linewidth}
            \includegraphics[width=\linewidth]{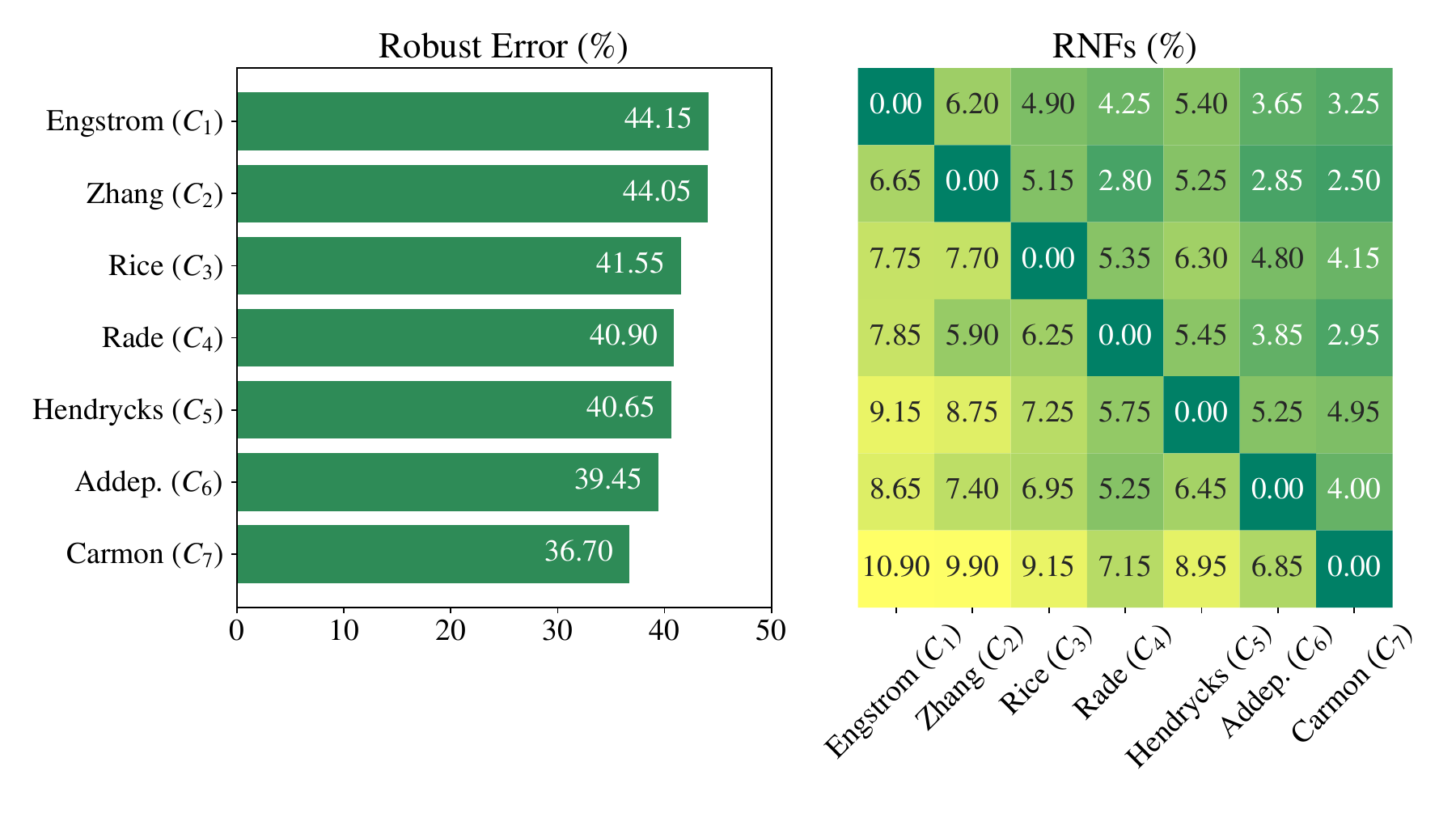}
            \caption{Regression of robustness, measured with RNF rates.}
            \label{fig:cifar10_robs}
        \end{subfigure}
        \begin{subfigure}{0.49\linewidth}
            \includegraphics[width=\linewidth]{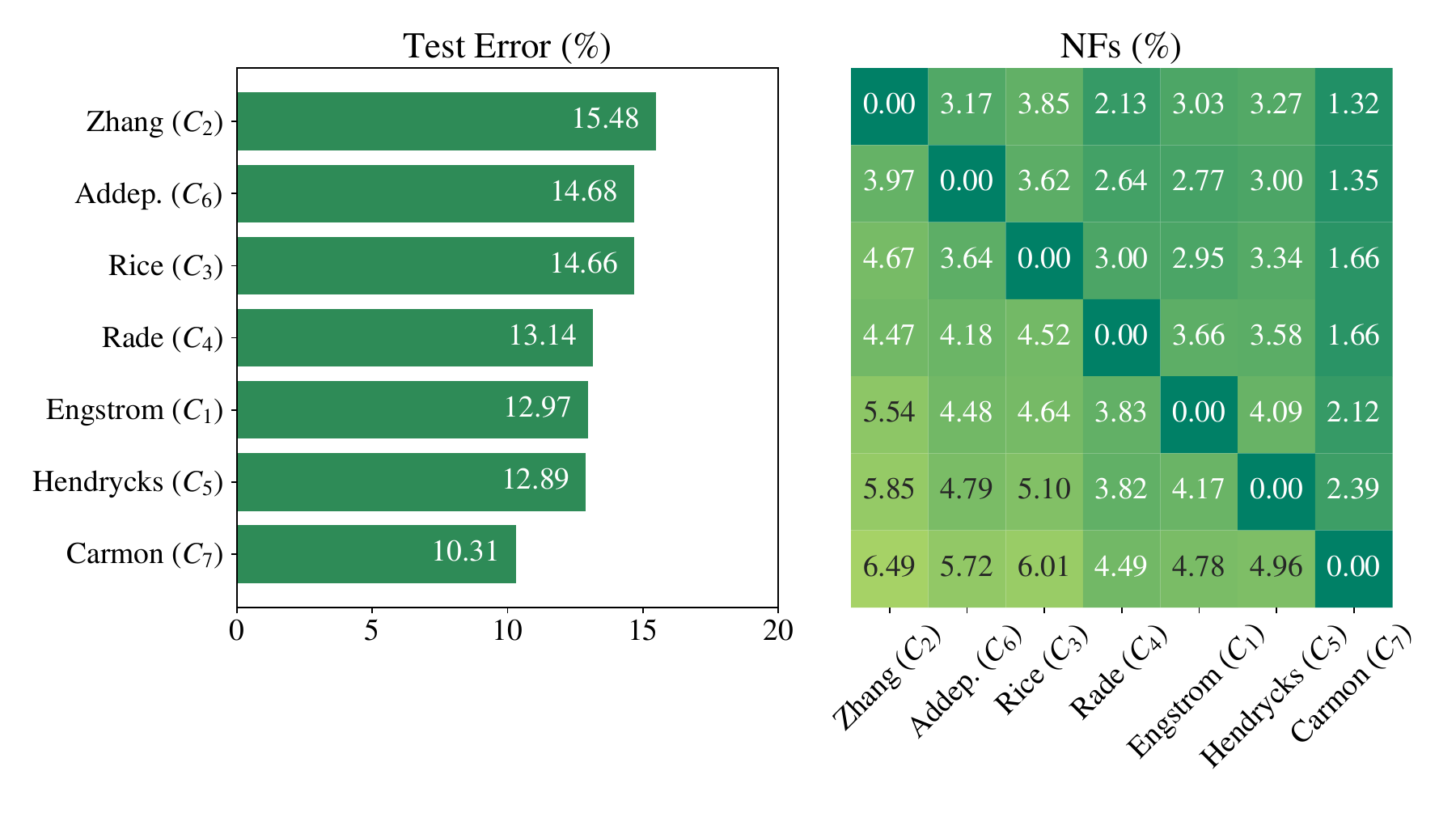}
            \caption{Regression of accuracy, measured with NF rates.}
            \label{fig:cifar10_acc}
        \end{subfigure}
    \caption{
    Regression of robustness and accuracy for \edit{the CIFAR-10 robust models $C_1, \ldots, C_7$}~\cite{Engstrom2019Robustness,Zhang2020Attacks,Rice2020Overfitting,Rade2021Helper,Hendrycks2019Using,Addepalli2021Towards,Carmon2019Unlabeled}.
    (a) \textit{Left:} Robust error (\%) of each model, sorted in descending order. \textit{Right:} RNFs (\%) attained when replacing old models (in \textit{rows}) with new ones (in \textit{columns}).
    (b) \textit{Left:} Test error (\%) of each model, sorted in descending order. \textit{Right:} NFs (\%) attained when replacing old models (in \textit{rows}) with new ones (in \textit{columns}). Values in the upper (lower) triangular matrices evaluate regression when the \textit{new} model has better (worse) average robustness/accuracy than the \textit{old} model.
    }
    \label{fig:cifar10_churn_matrix}
\end{figure*}

\myparagraph{Performance Metrics.} 
\edit{We consider four relevant metrics to evaluate the given methods:} (i) the \textit{test error}, \ie
the percentage of misclassified (clean) test samples;
(ii) the \textit{robust error}, \ie 
the percentage of misclassified adversarial examples;
(iii) the fraction of negative flips (NFs, \autoref{eq:nf}), to quantify the regression of accuracy; and (iv) the fraction of robustness negative flips (RNFs, \autoref{eq:rnf}), to evaluate the regression of robustness.
To evaluate the robust error and RNF rates, we optimize adversarial examples against each model using the AutoPGD~\cite{croce2020reliable} implementation from \textit{adversarial-library},\footnote{Available at \url{https://github.com/jeromerony/adversarial-library}.} with the \textit{Difference-of-Logit-Ratio} (DLR) loss and $N=50$ iterations.

\myparagraph{Hardware.} All the experiments have been run on a workstation with an Intel\textsuperscript{\textregistered} Xeon\textsuperscript{\textregistered} Gold 5217 CPU with 32 cores (3.00 GHz), 192 GB of RAM, and two Nvidia\textsuperscript{\textregistered} RTX A6000 GPUs with 48 GB of memory each.

\subsection{Evaluating Regression of Robustness}
\label{subsec:existence_rnf}

The first experiment presented here aims to empirically show that, once updated, machine-learning models are affected not only by a regression of accuracy, as shown in~\cite{yan2021positive}, but also by a regression of \textit{robustness}.
To show this novel phenomenon, as described in \autoref{subsec:exp_setup}, we first rank the \edit{CIFAR-10} models \edit{$C_1, ..., C_7$} from the least to the most robust, based on the robust error evaluated on the $2,000$ test samples perturbed with AutoPGD.
The models, along with their robust errors, are reported in order in \autoref{fig:cifar10_robs} (left).
Then, assuming a na\"ive update strategy that just replaces the old model with the new one, without performing any fine-tuning of the latter, we evaluate regression of robustness for all possible model pairs, as shown in \autoref{fig:cifar10_robs} (right).
Each cell of this matrix reports the regression of robustness (\ie the fraction of RNFs) induced by the new model (reported in the corresponding column) when replacing the old one (reported in the corresponding row).
The upper (lower) triangular matrix represents the cases where the new model has a better (lower) average robust error than the old one.
The diagonal corresponds to replacing the model with itself; thus, no regression is observed.
Even in the ideal case where machine-learning models are updated with more robust ones (\ie in the upper triangular matrix), the RNF rate ranges from roughly $2\%$ to $6\%$.
This implies that an increment in average robustness does not guarantee that \edit{previously-ineffective adversarial attacks on some samples remain unsuccessful after update.}

We perform the same analysis on \edit{$C_1, \ldots, C_7$} to quantify also their regression of accuracy.
We thus re-order the models according to their reported test errors, as shown in \autoref{fig:cifar10_acc} (left).
We then consider again all the possible model pairs to simulate different model updates, and report the NFs corresponding to each na\"ive model update (\ie by replacement) in \autoref{fig:cifar10_acc} (right). 
Not surprisingly, all updates induce a regression of accuracy, even when the new models are more accurate than the old ones (\ie in the upper triangular matrix), as also already shown by Yan et al.~\cite{yan2021positive}.

We can thus conclude that while updating machine-learning models may be beneficial to improve their average accuracy or robustness, this practice may induce a significant regression of both metrics when evaluated sample-wise.

\begin{table}
\centering
\caption{
Results for PCT, PCAT, and RCAT on the $10$ \edit{CIFAR-10} model updates with the highest RNF rates.
In each update, denoted with ($C_i, C_j$), the baseline model $C_i$ is replaced with $C_j$ using the \naive (replacement) strategy, or by fine-tuning $C_j$ with PCT, PCAT, and RCAT. For each update, we evaluate the test error, \anfs, robust error, and \rnfs (in \%), and highlight in \textbf{bold} the method achieving the best trade-off between NFs and RNFs, \ie the lowest value of their sum.
}
\label{tab:ft_results}
\resizebox{0.95\columnwidth}{!}{
\begin{tabular}{l l c c c c c}
\toprule
&           &        \textbf{Test Error} &           \textbf{NFs} &       \textbf{Robust Error} &           \textbf{RNFs} \\

\midrule \multirow{5}{*}{\rotatebox[origin=c]{90}{(1)}} 
\multirow{5}{*}{\rotatebox[origin=c]{90}{($C_3, C_5$)}} & \base &           14.66 &              - &           41.55 &              - \\
& \naive &           12.89 &           3.34 &           40.65 &           6.30 \\ \cmidrule{2-6} 
& \pct &    {6.77} &   {1.20} &           66.10 &          26.10 \\
& \pctat &            9.80 &           2.10 &           47.55 &          10.95 \\
& \textbf{\mixmseat} &          \textbf{ 11.25} &           \textbf{2.30} &   \textbf{40.45} &   \textbf{6.00} \\

\midrule \multirow{5}{*}{\rotatebox[origin=c]{90}{(2)}} 
\multirow{5}{*}{\rotatebox[origin=c]{90}{($C_4, C_5$)}} & \base &           13.14 &              - &           40.90 &              - \\
& \naive &           12.89 &           3.58 &   {40.65} &   {5.45} \\ \cmidrule{2-6} 
& \pct &    {6.16} &   {1.02} &           72.95 &          32.35 \\
& \pctat &            8.69 &           1.65 &           47.20 &           9.60 \\
& \textbf{\mixmseat} &           \textbf{10.82} &           \textbf{1.76} &           \textbf{41.55} &          \textbf{ 5.60} \\

\midrule \multirow{5}{*}{\rotatebox[origin=c]{90}{(3)}} 
\multirow{5}{*}{\rotatebox[origin=c]{90}{($C_1, C_5$)}} & \base &           12.97 &              - &           44.15 &              - \\
& \naive &           12.89 &           4.09 &           40.65 &           5.40 \\ \cmidrule{2-6} 
& \pct &    {8.01} &   {1.88} &           55.70 &          14.40 \\
& \pctat &            9.34 &           2.26 &           47.90 &           8.80 \\
& \textbf{\mixmseat} &           \textbf{10.75} &          \textbf{2.52} &   \textbf{40.50} &   \textbf{4.55} \\

\midrule \multirow{5}{*}{\rotatebox[origin=c]{90}{(4)}} 
\multirow{5}{*}{\rotatebox[origin=c]{90}{($C_3, C_4$)}} & \base &           14.66 &              - &           41.55 &              - \\
& \naive &           13.14 &           3.00 &   {40.90} &           5.35 \\ \cmidrule{2-6} 
& \pct &    {8.48} &   {1.37} &           51.30 &          12.45 \\
& \pctat &            9.99 &           1.72 &           44.90 &           7.80 \\
& \textbf{\mixmseat} &           \textbf{11.04} &           \textbf{1.68} &            \textbf{40.90} &   \textbf{4.60} \\

\midrule \multirow{5}{*}{\rotatebox[origin=c]{90}{(5)}} 
\multirow{5}{*}{\rotatebox[origin=c]{90}{($C_2, C_5$)}} & \base &           15.48 &              - &           44.05 &              - \\
& \naive &           12.89 &           3.27 &           40.65 &           5.25 \\ \cmidrule{2-6} 
& \pct &    {5.35} &   {1.12} &           84.05 &          40.35 \\
& \pctat &            8.81 &           1.65 &           47.00 &           8.35 \\
& \textbf{\mixmseat} &           \textbf{11.81} &           \textbf{2.27} &   \textbf{40.45} &   \textbf{4.85} \\

\midrule \multirow{5}{*}{\rotatebox[origin=c]{90}{(6)}} 
\multirow{5}{*}{\rotatebox[origin=c]{90}{($C_2, C_3$)}} & \base &           15.48 &              - &           44.05 &              - \\
& \naive &           14.66 &           3.85 &   {41.55} &           5.15 \\ \cmidrule{2-6} 
& \pct &   {10.59} &   {1.76} &           46.90 &           7.70 \\
& \textbf{\pctat} &           \textbf{14.53} &           \textbf{2.41} &           \textbf{42.10} &   \textbf{3.15} \\
& \mixmseat &           14.25 &           3.29 &           41.90 &           5.05 \\

\midrule \multirow{5}{*}{\rotatebox[origin=c]{90}{(7)}} 
\multirow{5}{*}{\rotatebox[origin=c]{90}{($C_5, C_7$)}} & \base &           12.89 &              - &           40.65 &              - \\
& \naive &           10.31 &           2.39 &           36.70 &           4.95 \\ \cmidrule{2-6} 
& \pct &    {6.47} &   {1.21} &           40.90 &           7.55 \\
& \pctat &            7.55 &           1.73 &           38.85 &           6.75 \\
& \textbf{\mixmseat} &           \textbf{8.30} &           \textbf{1.45} &   \textbf{36.50} &   \textbf{4.65} \\

\midrule \multirow{5}{*}{\rotatebox[origin=c]{90}{(8)}} 
\multirow{5}{*}{\rotatebox[origin=c]{90}{($C_3, C_7$)}} & \base &           14.66 &              - &           41.55 &              - \\
& \naive &           10.31 &           1.66 &   {36.70} &   {4.15} \\ \cmidrule{2-6} 
& \pct &    {7.40} &   {0.97} &           41.10 &           6.25 \\
& \pctat &            8.74 &           1.24 &           39.00 &           4.75 \\
& \textbf{\mixmseat} &           \textbf{ 8.80} &           \textbf{1.21} &           \textbf{37.55} &           \textbf{4.35} \\

\midrule \multirow{5}{*}{\rotatebox[origin=c]{90}{(9)}} 
\multirow{5}{*}{\rotatebox[origin=c]{90}{($C_6, C_7$)}} & \base &           14.68 &              - &           39.45 &              - \\
& \naive &           \textbf{10.31} &           \textbf{1.35} &   \textbf{36.70} &   \textbf{4.00} \\ \cmidrule{2-6} 
& \pct &    {7.16} &   {1.06} &           42.75 &           8.25 \\
& \pctat &            7.93 &           1.34 &           40.15 &           6.95 \\
& \mixmseat &            8.92 &           1.10 &           37.80 &           5.00 \\

\midrule \multirow{5}{*}{\rotatebox[origin=c]{90}{(10)}} 
\multirow{5}{*}{\rotatebox[origin=c]{90}{($C_1, C_7$)}} & \base &           12.97 &              - &           44.15 &              - \\
& \naive &           10.31 &           2.12 &   {36.70} &           3.25 \\ \cmidrule{2-6} 
& \pct &    {8.01} &           1.22 &           41.50 &           4.30 \\
& \pctat &            9.31 &           1.33 &           40.95 &           3.70 \\
& \textbf{\mixmseat} &            \textbf{8.41} &   \textbf{1.11} &           \textbf{37.15} &   \textbf{2.80} \\

\bottomrule
\hline 
\end{tabular}
}
\end{table}

\begin{table}
\centering
\caption{
\edit{Results for PCT, PCAT, and RCAT on the $4$ ImageNet model updates.
In each update, denoted with ($I_i,I_j$), the baseline model $I_i$ is replaced with $I_j$ using the \naive (replacement) strategy, or by fine-tuning $I_j$ with PCT, PCAT, and RCAT. 
For each update, we evaluate the test error, \anfs, robust error, and \rnfs (in \%), and highlight in \textbf{bold} the method achieving the best trade-off between NFs and RNFs, \ie the lowest value of their sum.
}
}
\label{tab:imagenet_oneshot}
\resizebox{0.99\columnwidth}{!}{
\begin{tabular}{l l c c c c c}
\toprule
&           &        \textbf{Test Error} &           \textbf{NFs} &       \textbf{Robust Error} &           \textbf{RNFs} \\

\midrule 
\multirow{5}{*}{\rotatebox[origin=c]{90}{($I_1, I_2$)}} & \base &           47.78 &              - &           73.48 &              - \\
& \naive &           37.66 &           2.22 &           66.98 &           2.88 \\ \cmidrule{2-6} 
& \pct &             34.92 &            2.32 &           76.54 &          7.62 \\
& \pctat &            38.26 &           2.66 &           65.38 &          2.32 \\
& \textbf{\mixmseat} &      \textbf{37.70} &           \textbf{1.96} &   \textbf{64.78} &   \textbf{2.00} \\

\midrule 
\multirow{5}{*}{\rotatebox[origin=c]{90}{($I_2, I_3$)}} & \base &           37.66 &              - &           66.98 &              - \\
& \naive &           31.22 &           2.84 &   57.22 &   2.14 \\ \cmidrule{2-6} 
& \pct &    28.42 &   2.94 &           83.52 &          17.82 \\
& \pctat &            32.58 &           4.14 &           60.74 &           3.96 \\
& \textbf{\mixmseat} &           \textbf{30.90} &           \textbf{2.24} &           \textbf{55.88} &          \textbf{1.62} \\

\midrule 
\multirow{5}{*}{\rotatebox[origin=c]{90}{($I_3, I_4$)}} & \base &           31.22 &              - &           57.22 &              - \\
& \naive &           23.60 &           2.76 &   44.08 &   2.20 \\ \cmidrule{2-6} 
& \pct &    19.96 &   2.16 &           75.12 &          20.50 \\
& \textbf{\pctat} &            \textbf{24.68} &           \textbf{2.28} &           \textbf{44.02} &           \textbf{1.62} \\
& \mixmseat &           23.12 &           2.26 &           43.08 &          1.82 \\

\midrule 
\multirow{5}{*}{\rotatebox[origin=c]{90}{($I_4, I_5$)}} & \base &           23.60 &              - &           44.08 &              - \\
& \naive &           22.12 &           2.74 &   40.64 &   2.54 \\ \cmidrule{2-6} 
& \pct &    17.70 &   2.00 &           80.58 &          37.22 \\
& \textbf{\pctat} &            \textbf{20.58} &           \textbf{2.44} &           \textbf{38.48} &           \textbf{1.92} \\
& \textbf{\mixmseat} &           \textbf{21.24} &           \textbf{2.30} &           \textbf{39.54} &          \textbf{2.06} \\

\bottomrule
\hline 
\end{tabular}
}
\end{table}

\subsection{Reducing Regression in Robust Image Classifiers}
\label{sect:ftuning}

After quantifying the non-negligible impact of \rnfs in model updates, we show how RCAT can tackle this issue, outperforming the competing model update strategies of PCT and PCAT.
To this end, we first select suitable model pairs that enable simulating updates in which the new model improves \textit{both} accuracy \textit{and} robustness \wrt the old one. This amounts to considering only $14$ model pairs (out of the overall $42$) \edit{for CIFAR-10}.
Among these $14$ cases, we exclude $4$ model pairs exhibiting an RNF rate lower than $3\%$, which results in retaining the $10$ cases with the highest RNFs listed in \autoref{tab:ft_results}.
\edit{Let us also recall that, when using the \naive update strategy, we just replace the old model $C_i$ with the new model $C_j$, and measure the corresponding NFs and RNFs. When using PCT, PCAT, and RCAT, we initialize the new model to be fine-tuned with $C_j$. For RCAT, we also set $\fsrc = C_j$.}
\edit{For ImageNet models, we consider the four pairs reported in \autoref{tab:imagenet_oneshot}, in which the new models always exhibit improved accuracy and robustness \wrt the old ones, with non-negligible NF and RNF rates of about 2-3\%.}

\myparagraph{Results for Model Updates.} 
The results in \autoref{tab:ft_results} \edit{and \autoref{tab:imagenet_oneshot}} report the performance metrics for each model update.\footnote{\edit{We also estimate their standard deviation using the bootstrap method with 1,000 resamplings, and find that it is negligible, being approximately 0.6\% for the test error, 0.25\% for NFs, 1\% for the robust error, and 0.5\% for RNFs, for all methods, with minor variations.}} It is clear that \rcat provides a better trade-off between the reduction of NFs and that of RNFs \wrt the competing approaches in almost all cases. RCAT indeed achieves lower values of the sum of NFs and RNFs, while PCT and PCAT mostly reduce NFs by improving accuracy at the expense of compromising robustness. 
More specifically, PCT almost always entirely compromises robustness, as it is not designed to preserve it after update. 
A paradigmatic example can be found in row 5 of \autoref{tab:ft_results}, where \pct recovers almost 10\% of test error \wrt the \naive strategy, lowering the \anfs \edit{close} to $1\%$, but increasing the robust error by \edit{approximately} $43\%$ and the RNF rate by \edit{almost} $35\%$.
Similar trends are also reported for ImageNet models; \eg considering the model update $(I_4, I_5)$, \pct reduces the test error by almost 5\% while increasing the robust error and RNFs by \edit{approximately} $40\%$ and $35\%$, respectively. \edit{Let us finally remark that we also compute the number of samples for which no adversarial example was found before update, but become misclassified after update, contributing to both NFs and RNFs as discussed in \autoref{sect:regression.rnfs}. These \emph{joint} negative flips are always less than approximately 0.03\% for CIFAR-10 and 0.2\% for ImageNet, thus being negligible overall.}

\myparagraph{Measuring Performance Improvements ($\Delta$-metrics).} 
To better highlight the differences among PCT, PCAT, and RCAT, for each update \edit{($M_i$, $M_j$) (being $M_i$ and $M_j$ two generic CIFAR-10 or ImageNet models)}, we compute their performance improvements \wrt the baseline \naive strategy using the following four $\Delta$-metrics:
\begin{itemize}
    \item $\Delta$ Test Error (\%), \ie the difference between the Test Error obtained by replacing $M_i$ with $M_j$ (\naive strategy) and that obtained using PCT, PCAT, or RCAT;
    \item $\Delta$ Robust Error (\%), \ie the difference between the Robust Error obtained by the \naive strategy and that obtained using PCT, PCAT, or RCAT;
    \item $\Delta$ NFs (\%), \ie the difference between the NF rate obtained by the \naive strategy and that obtained using PCT, PCAT, or RCAT; and
    \item $\Delta$ RNFs (\%), \ie the difference between the RNF rate obtained by the \naive strategy and that obtained using PCT, PCAT, or RCAT.
\end{itemize}
Accordingly, \textit{positive} values of the $\Delta$-metrics represent an improvement \wrt \naive model replacement, \ie a reduction of the test error, the robust error, NFs, and RNFs.

\myparagraph{Analysis with $\Delta$-metrics.} \edit{The results are reported in 
\autoref{tab:mean_diff}}, along with the values of the $\Delta$-metrics averaged on the model updates listed in \autoref{tab:ft_results} \edit{and \autoref{tab:imagenet_oneshot}}.
\edit{PCT provides the largest improvement in accuracy for both CIFAR-10 and ImageNet models, and a good reduction of NFs, at the expense of significantly worsening adversarial robustness (-15.14\%/-22.50\% for CIFAR-10 and ImageNet, respectively) and RNFs (-11.05\%/-14.85\%).}
\edit{PCAT works slightly better for CIFAR-10 models, improving accuracy and NFs by +2.59\% and +1.12\% on average, but decreasing the robust error and RNF rate by -4.38\% and -2.16\% on average. For ImageNet models, PCAT performs similarly to the \naive update strategy.}
\edit{The proposed \rcat method, instead, finds a better accuracy-robustness trade-off; in particular, it only slightly improves accuracy (+1.62\%/+0.39\% on average) and NFs (+1\%/+0.55\% on average) without significantly affecting robustness, while improving RNFs (+0.18\%/+0.66\% on average). We also note that for ImageNet models, \rcat is the best performing in terms of NFs (+0.55\% on average).}\footnote{It is worth remarking here that, in general, the additional constraint of reducing NFs/RNFs in the considered update strategies may cause the test/robust error to increase (\eg RCAT slightly worsens the robust error, on average, while reducing the RNF rate). This may become more evident in the case of more accurate/robust models.}

\begin{table}
\centering
\caption{
Mean values of the $\Delta$-metrics, averaged on all \edit{CIFAR-10} (ImageNet) model updates in \autoref{tab:ft_results} \edit{(\autoref{tab:imagenet_oneshot})}. 
Positive values mean better results, \ie lower test/robust error and NF/RNF rate \wrt to the \naive strategy (\ie model replacement). 
\edit{For each $\Delta$-metric and dataset, we highlight in \textbf{bold} the highest improvement among the three methods.}
}
\label{tab:mean_diff}
\resizebox{0.99\columnwidth}{!}{
\begin{tabular}{l l c c c c c}
\toprule
 & & \textbf{$\Delta$ Test Err.} & \textbf{$\Delta$ NFs} & \textbf{$\Delta$ Rob. Err.} &  \textbf{$\Delta$ RNFs} \\
\midrule
\multirow{3}{*}{CIFAR-10}  
& \pct      &      \textbf{+4.62} &  \textbf{+1.58} &     -15.14 &  -11.05 \\
& \pctat    &      +2.59 &  +1.12 &      -4.38 &   -2.16 \\
& \mixmseat &      +1.62 &  +1.00 &      \textbf{-0.29} &    \textbf{+0.18} \\

\midrule
\multirow{3}{*}{ImageNet}  
& \pct      &      \textbf{+3.08} &  +0,30 &     -22.50 &  -14.85 \\
& \pctat    &      +0.01 &  -0.17 &      +0.62 &   -0.03 \\
& \mixmseat &      +0.39 &  \textbf{+0.55} &      \textbf{+1.61} &    \textbf{+0.66} \\

\bottomrule
\end{tabular}
}
\end{table}

\subsection{\edit{Reducing Regression over Sequential Updates}}
\label{sect:sequential}

\edit{We have considered so far a single-update setting in which the \textit{old} model is replaced by the \textit{new} one, using a given update method (being it \naive, \pct, \pctat, or \rcat). In this section, instead, we consider the case in which we perform multiple sequential updates to evaluate the behavior of the update strategies when used iteratively.
To this end, in \autoref{tab:imagenet_sequential} we report the results of \rcat, \pct, and \pctat when used to update the ImageNet models $I_1, \ldots I_5$ sequentially. In this experiment, we set $(\alpha, \beta) = (1, 2)$ for \pct and \pctat, and $(0.5, 0.4)$ for \rfat, and retain these values for all the sequential updates.}
\edit{PCT significantly reduces the test error while containing NFs across the four updates, but clearly worsens the robustness and RNFs. PCAT, relying on adversarial training, mitigates this issue, but still exhibits NF and RNF rates around 3\%. RCAT achieves the best trade-off again in this setting also, with comparable test and robust errors to those of PCAT, while reporting consistently lower NF and RNF rates of about 2\%. This confirms the capability of \rcat to also deal with the case of multiple sequential updates.}

\begin{table}
\centering
\caption{
\edit{Multiple sequential updates on ImageNet models using PCT, PCAT, and RCAT. We show the evolution of test error, NFs, robust error, and RNFs (in \%), when applying each method sequentially, \ie using the newly-trained model as a baseline for the next update.}
}
\label{tab:imagenet_sequential}
\setlength{\tabcolsep}{5pt}
\resizebox{0.97\columnwidth}{!}{
\begin{tabular}{l l c c c c}
\toprule
& & \textbf{Test Error} &  \textbf{NFs} &  \textbf{Robust Error} &  \textbf{RNFs} \\

\midrule 
\multirow{4}{*}{{{\pct}}}        
& $(I_1, I_2)$ & 35.50 &	2.28 &	74.30  &	6.14  \\
& $(I_2, I_3)$ & 28.90  &	3.16  &	84.08  &	18.32  \\
& $(I_3, I_4)$ & 22.88  &	2.68  &	69.72  &	15.30  \\
& $(I_4, I_5)$ & 19.64  &	2.84  &	70.28  &	27.22  \\

\midrule 
\multirow{4}{*}{{{\pctat}}}  
& $(I_1, I_2)$ & 38.62  &	2.58  &	65.12  &	2.20  \\
& $(I_2, I_3)$ & 33.36  &	4.76  &	60.56  &	3.98  \\
& $(I_3, I_4)$ & 25.78  &	2.96  &	44.42  &	1.94  \\
& $(I_4, I_5)$ & 21.86  &	3.32  &	40.18  &	3.44  \\

\midrule 
\multirow{4}{*}{{{\rcat}}}  
& $(I_1, I_2)$ & 37.98  &	1.84  &	64.76  &	1.96  \\
& $(I_2, I_3)$ & 31.00  &	2.30  &	55.64  &	1.58  \\
& $(I_3, I_4)$ & 23.86  &	2.30  &	43.28  &	1.66  \\
& $(I_4, I_5)$ & 21.48  &	2.42  &	39.22  &	2.02  \\

\bottomrule
\hline 
\end{tabular}
}
\end{table}

To summarize, performing model updates that retain high accuracy and robustness with low regression of both metrics is challenging. PCT tends to recover accuracy, at the expense of worsening robustness, as it was not originally designed to also deal with adversarial robustness. PCAT and RCAT find better accuracy-robustness trade-offs, with RCAT outperforming PCAT. In particular, RCAT provides comparable improvements in accuracy while preserving \edit{or even improving} robustness and, at the same time, minimizing the sum of NFs and RNFs, thereby effectively reducing regression of both accuracy and robustness.

\section{Related Work}
\label{sect:related}
We now discuss methodologies that are related to our work. 
\myparagraph{Continual Learning.} Our work has ties to the research field of \emph{continual learning} (CL)~\cite{chen2018lifelong,de2021continual}, which aims to continuously retrain machine-learning models to deal with new classes, tasks, and domains, while only slightly adapting the initial architecture.
However, it has been shown that such continuous updates render CL techniques sensitive to \textit{catastrophic forgetting}~\cite{kirkpatrick2017overcoming,toneva2018empirical}, \ie to forget previously-learned classes, tasks, or domains, while retaining good performance only on more recent data.
Several techniques have been proposed to mitigate this issue, including: \emph{replay methods}~\cite{ahn2021ss, chaudhry2019continual}, which replay representative samples from past data while retraining models on subsequent tasks; \textit{regularization terms}~\cite{ahn2019uncertainty, wang2021training}, which promote solutions with similar weights to those of the given model; and \textit{parameter isolation}~\cite{mallya2018packnet, serra2018overcoming}, which separates weights attributed to the different task to be learned.
However, as also explained in~\cite{yan2021positive}, we argue that quantifying the regression of models is not strictly connected to CL, as we are neither considering the inclusion of new tasks to learn nor the adaptation to the evolution of the data distribution.
Furthermore, we do not restrict our update policy to maintain the same architecture, but we permit its replacement with a better one in terms of both accuracy and robustness.

\myparagraph{Backward Compatibility.}
Our methodology can be included among the so-called \emph{backward-compatible}~\cite{zhao2022elodi, trauble2021backward, shen2020towards} learning approaches, which focus on providing updates of machine learning models that can interchangeably replace previous versions without suffering loss in accuracy.
This can be achieved by: (i) learning an invariant representation for newer and past data~\cite{shen2020towards}; using different weights when predicting specific samples~\cite{srivastava2020empirical}; and (iii) estimating which samples should be re-evaluated as their labels might be incorrect.
However, unlike \rfat, all these techniques only focus on accuracy, ignoring the side-effects on the regression of robustness, as the dramatic drop in performance we have observed when using \pct.

\section{Conclusions and Future Work}
\label{sect:concl}

Modern machine-learning systems demand frequent model updates to improve their average performance. To this end, such systems often exploit more powerful architectures and additional data to update the current models. However, it has been shown that model updates can induce a perceived regression of accuracy in the end users, as the new model may commit mistakes that the previous one did not make. The corresponding samples misclassified by the new model are referred to as \emph{negative flips} (\anfs).
In this work, we show that \anfs are not the sole regression that machine-learning models can face. Model updates can indeed also cause a significant regression of \textit{adversarial robustness}.
This means that, even if the average robustness of the updated model is higher, some adversarial examples that were not found for certain inputs against the previous model can be found against the new one. We refer to these samples as \emph{robustness negative flips} (\rnfs).
To address this issue, we propose a novel algorithm named \rfat, based on adversarial training, and theoretically show that our methodology provides a statistically-consistent estimator, without affecting the usual convergence rate of $O(\sfrac{1}{\sqrt{n}})$.
We empirically show the existence of \rnf while updating robust image classification models, and compare the performance of our \rfat approach with \pctat, \ie the adversarially-robust version of \pct.
The results highlight that \rfat better handles the regression of robustness, by reducing the number of \rnf and retaining or even improving the same performance as the previous model.

\edit{While we have not considered cases in which the data can also change over time, along with the models, we argue that our methodology can be readily applied also under these more challenging conditions. We will better investigate this aspect in future work, considering different application domains in which data quickly evolves over time, demanding frequent model updates, such as in the case of spam and malware detection. 
Within this context, we also plan to improve the proposed approach by studying the effect of different loss functions and regularizers, along with the investigation of better model selection methods and the implications of the no-free-lunch theorem~\cite{wolpert1996lack}, in particular, related to the trade-off between accuracy and robustness in non-stationary settings.}
To conclude, we firmly believe that this first work can set up a novel line of research, which will educate practitioners to evaluate and mitigate the different types of regression that might be faced when dealing with machine-learning model updates.

\section*{Acknowledgments}
This research has been partially supported by the Horizon Europe projects ELSA (grant agreement no. 101070617) 
and CoEvolution (grant agreement no. 101168560); 
by SERICS (PE00000014) and FAIR (PE00000013) under the MUR NRRP funded by the EU-NGEU; and by the EU-NGEU National Sustainable Mobility Center (CN00000023), MUR Decree n. 1033—17/06/2022 (Spoke 10). This work was conducted while Daniele Angioni was enrolled in the Italian National Doctorate on AI run by Sapienza University of Rome in collaboration with the University of Cagliari.



\begin{IEEEbiography}[{\includegraphics[width=1in,height=1.25in,clip,keepaspectratio]{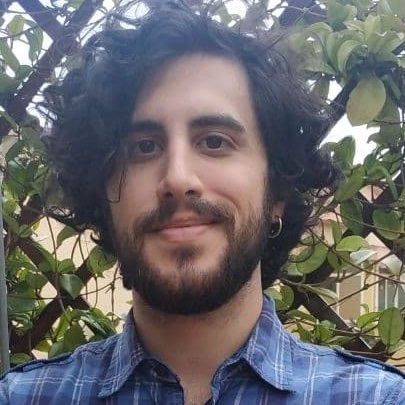}}]{Daniele Angioni} is a Postdoctoral Researcher at the University of Cagliari, Italy. He received his MSc degree in Electronic Engineering (with honors) in 2021, and his Ph.D. in Artificial Intelligence in January 2025 as part of the Italian National PhD Programme. His work focuses on machine learning security in real-world scenarios, including concept drift in malware detection, physical adversarial attacks on image classifiers, and performance regressions introduced by model updates. He serves as a reviewer for machine learning and cybersecurity journals (Pattern Recognition, Machine Learning, IEEE TIFS) and conferences (ESANN, AISec).
\end{IEEEbiography}

\begin{IEEEbiography}[{\includegraphics[width=1in,height=1.25in,clip,keepaspectratio]{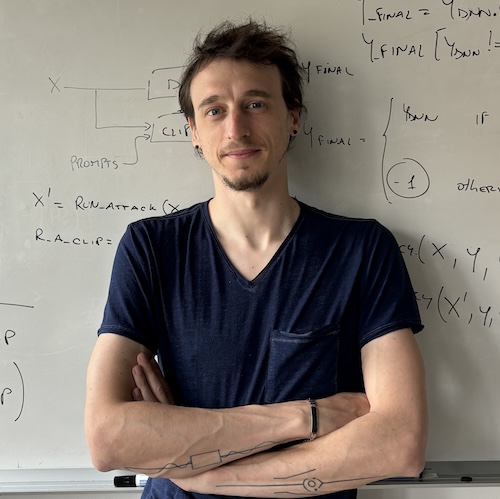}}]{Luca Demetrio}(MSc 2017, Ph.D. 2021) is Assistant Professor at the University of Genoa.
He serves as Associated Editor for Pattern Recognition, and as reviewer for top-tier conferences like USENIX, ICLR, NeurIPS and journals like IEEE TIFS and Computer \& Security.
He is currently studying the security of Windows malware detectors implemented with Machine Learning techniques, and he is first author of papers published in top-tier journals (ACM TOPS, IEEE TIFS).
\end{IEEEbiography}

\begin{IEEEbiography}[{\includegraphics[width=1in,height=1.25in,clip,keepaspectratio]{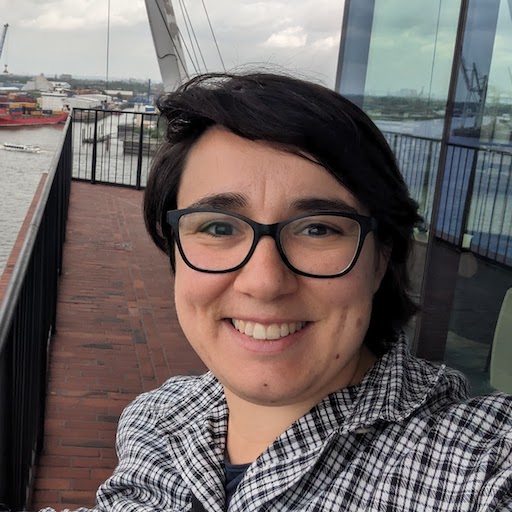}}]{Maura Pintor} is an Assistant Professor at the University of Cagliari, Italy. She received her PhD in Electronic and Computer Engineering (with honors) in 2022 from the University of Cagliari. Her research interests include adversarial machine learning and robustness evaluations of ML models, with applications in cybersecurity and computer vision. 
She serves as AC for NeurIPS, as AE for the journal Pattern Recognition, and as a PC member for several top-tier conferences (ACM CCS, ECCV, ICLR, ICCV, and for the journals IEEE TIFS, IEEE TIP, IEEE TDSC, IEEE TNNLS, ACM TOPS. She is co-chair of the ACM Workshop on Artificial Intelligence and Security (AISec), co-located with ACM CCS. She is a member of the Information Forensics and Security (IFS) Technical Committee
of the IEEE SPS.
\end{IEEEbiography}

\begin{IEEEbiography}[{\includegraphics[width=1in,height=1.25in,clip,keepaspectratio]{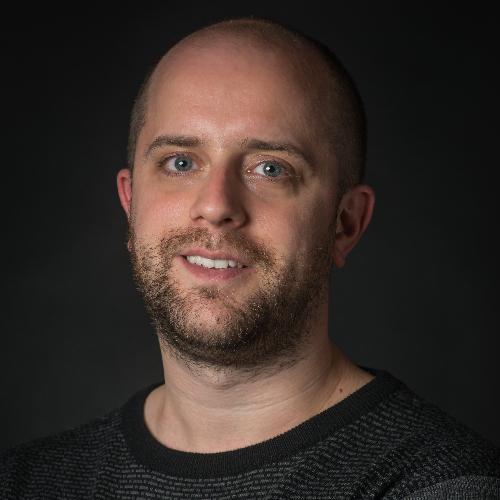}}]{Luca Oneto} was born in 1986 in Rapallo, Italy. He  completed his BSc and MSc in Electronic Engineering at the University of Genoa in 2008 and 2010, respectively. In 2014, he earned his PhD in Computer Engineering from the same institution. From 2014 to 2016, he worked as a Postdoc in Computer Engineering at the University of Genoa, where he then served as an Assistant Professor from 2016 to 2019. Luca co-founded the company ZenaByte s.r.l. in 2018. In 2019, he became an Associate Professor in Computer Science at the University of Pisa, and from 2019 to 2024, he held the position of Associate Professor in Computer Engineering at the University of Genoa. Currently, he is a Full Professor in Computer Engineering at the University of Genoa. He has been coordinator and local responsible in numerous industrial, H2020, and Horizon Europe projects. He has received prestigious recognitions, including the Amazon AWS Machine Learning Award and the Somalvico Award for the best young AI researcher in Italy. His primary research interests lie in Statistical Learning Theory and Trustworthy AI. Additionally, he focuses on data science, utilizing and improving cutting-edge machine learning and AI algorithms to tackle real-world problems.
\end{IEEEbiography}

\begin{IEEEbiography}[{\includegraphics[width=1in,height=1.25in,clip,keepaspectratio]{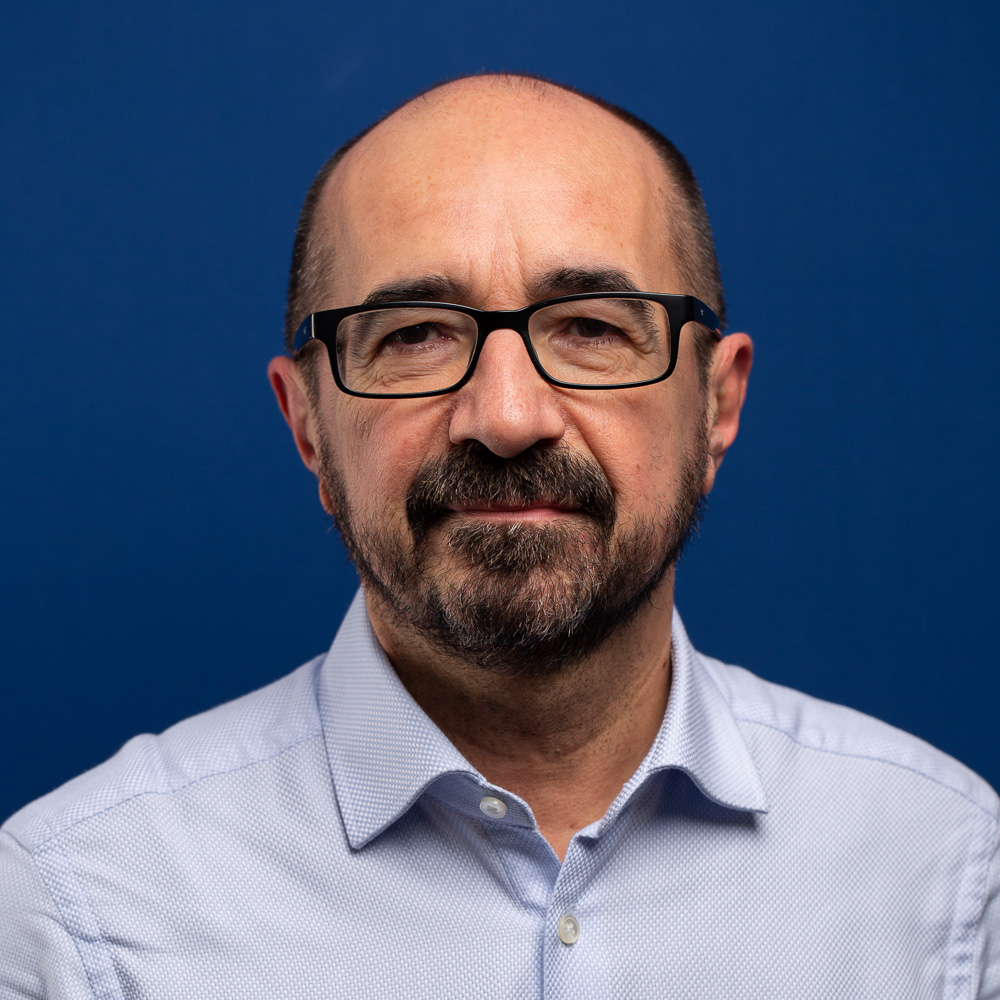}}]{Davide Anguita} received the ``Laurea'' degree in Electronic Engineering and a Ph.D. degree in Computer Science and Electronic Engineering from the University of Genoa, Italy, in 1989 and 1993, respectively. After working as a Research Associate at the International Computer Science Institute, Berkeley, CA, on special-purpose processors for neurocomputing, he returned to the University of Genoa. He is currently Full Professor of Computer Engineering with the Department of Informatics, BioEngineering, Robotics, and Systems Engineering (DIBRIS). His current research focuses on the theory and application of kernel methods and artificial neural networks.
\end{IEEEbiography}

\begin{IEEEbiography}[{\includegraphics[width=1in,height=1.25in,clip,keepaspectratio]{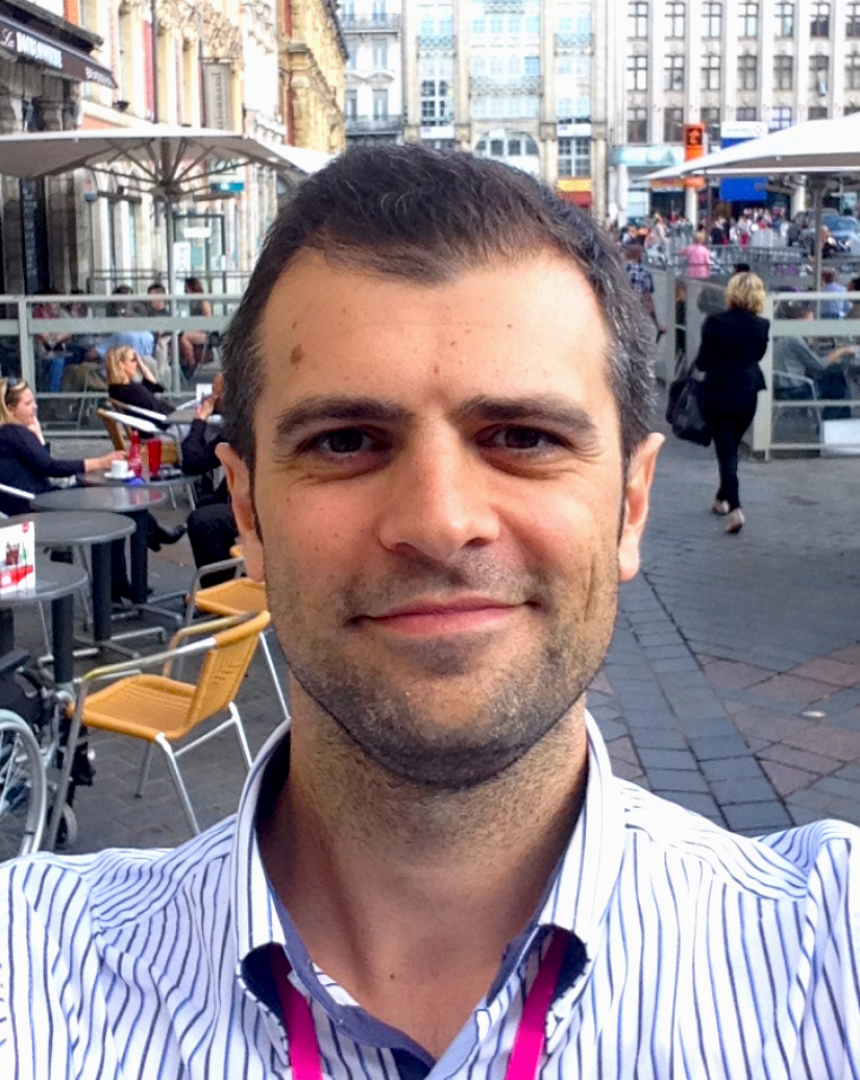}}]{Battista Biggio} (MSc 2006, PhD 2010) is a Full Professor of Computer Engineering at the University of Cagliari, Italy. He has provided pioneering contributions in machine learning security, playing a leading role in this field. His seminal paper ``Poisoning Attacks against Support Vector Machines'' won the prestigious 2022 ICML Test of Time Award. He has managed several research projects, and served as Area Chair for the IEEE Symposium on SP, and NeurIPS. He was IAPR TC1 Chair (2016-2020), and Associate Editor of IEEE TNNLS, IEEE CIM, and Elsevier Pattern Recognition (PRJ). He is now Associate Editor-in-Chief of PRJ. He is Fellow of IEEE and AAIA, Senior Member of ACM, and Member of AAAI, IAPR, and ELLIS. 
\end{IEEEbiography}

\begin{IEEEbiography}[{\includegraphics[width=1in,height=1.25in,clip,keepaspectratio]{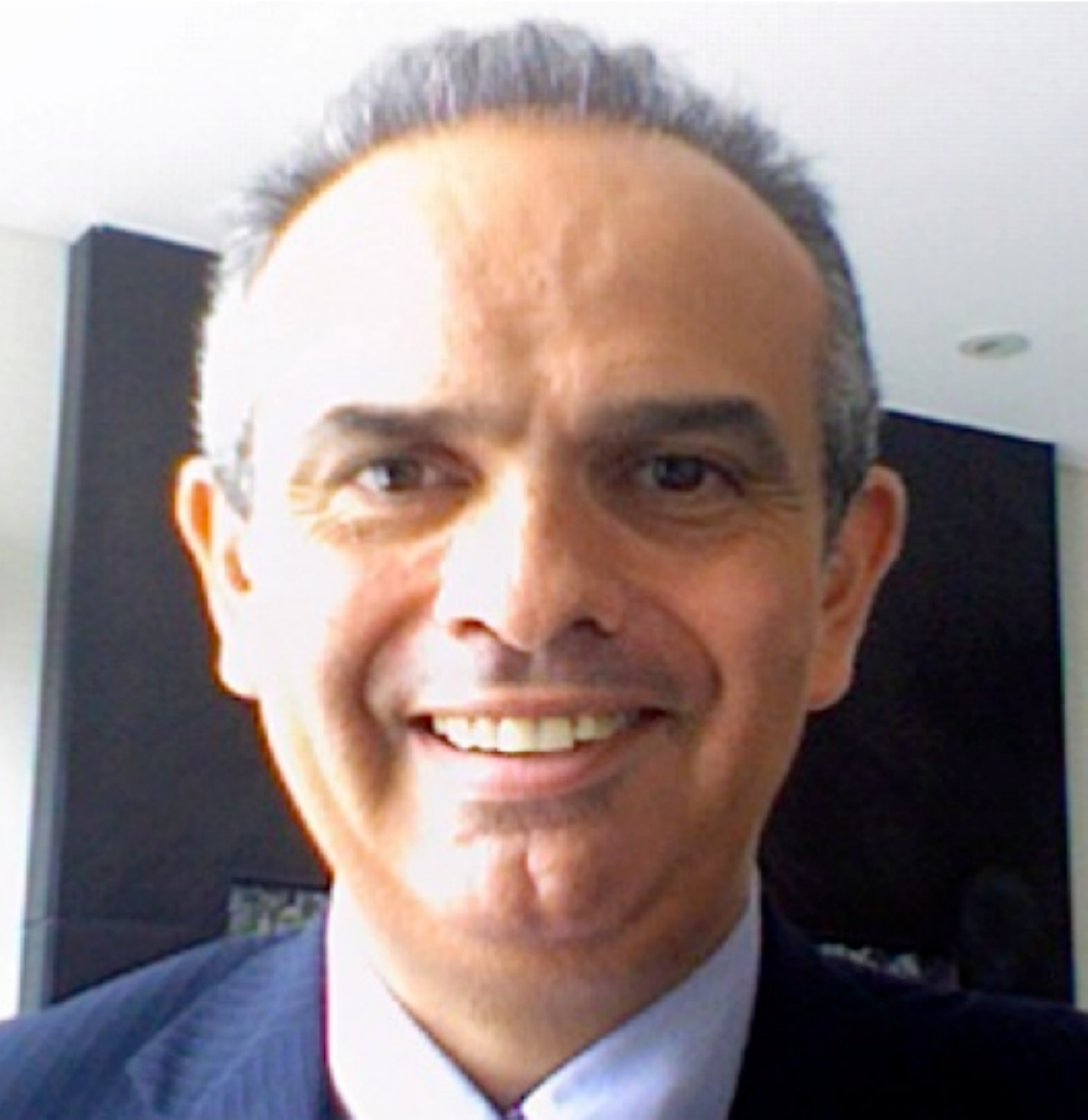}}]{Fabio Roli} is a Full Professor of Computer Engineering at the Universities of Genoa and Cagliari, Italy. He is Director of the sAIfer Lab, a joint lab between the Universities of Genoa and Cagliari on Safety and Security of AI. Fabio Roli’s research over the past thirty years has addressed the design of machine learning systems in the context of real security applications. He has provided seminal contributions to the fields of ensemble learning and adversarial machine learning and he has played a leading role in the establishment and advancement of these research themes. He is a recipient of the Pierre Devijver Award for his contributions to statistical pattern recognition. He has been appointed Fellow of the IEEE, Fellow of the International Association for Pattern Recognition, and Fellow of the Asia-Pacific Artificial Intelligence Association.
\end{IEEEbiography}

\end{document}